\documentclass[11pt]{article}

\usepackage[preprint]{acl}
\usepackage{times}
\usepackage{latexsym}
\usepackage[utf8]{inputenc} % allow utf-8 input
\usepackage[T1]{fontenc}    % use 8-bit T1 fonts
\usepackage{booktabs}       % professional-quality tables
\usepackage{amsfonts}       % blackboard math symbols
\usepackage{nicefrac}       % compact symbols for 1/2, etc.
\usepackage{microtype}      % microtypography
\usepackage{inconsolata}
\usepackage{xcolor}         % colors
\usepackage{enumitem}

\usepackage{graphicx}
\usepackage{subfigure}

\usepackage{amsmath}
\usepackage{amssymb}
\usepackage{mathtools}
\usepackage{amsthm}
\usepackage{algorithm}
\usepackage{algpseudocode}
\usepackage{multirow}
\usepackage{bm}
\usepackage[capitalize,noabbrev]{cleveref}

\theoremstyle{plain}
\newtheorem{theorem}{Theorem}[section]
\newtheorem{proposition}[theorem]{Proposition}
\newtheorem{lemma}[theorem]{Lemma}
\newtheorem{corollary}[theorem]{Corollary}
\theoremstyle{definition}

\theoremstyle{remark}
\newtheorem{remark}[theorem]{Remark}

\usepackage{xspace}

\newcommand{\ours}{SPARE\xspace}

\title{SPARE: Sparsity Allocation via Routing Entropy for Efficient MoEs}
\author{
  \textbf{Yifan Yang \quad Hanling Yi \quad Dongquan Yang \quad Feng Lin} \\
  \textbf{Rong Xiao \quad Aijun Li \quad Xiaotian Yu} \\
  Intellifusion Inc.
}
\hypersetup{
  pdftitle={SPARE: Sparsity Allocation via Routing Entropy for Efficient MoEs},
  pdfauthor={Yifan Yang, Hanling Yi, Dongquan Yang, Feng Lin, Rong Xiao, Aijun Li, Xiaotian Yu}
}

\begin{document}

\maketitle

\begin{abstract}
Mixture-of-Experts (MoE) architectures scale model capacity efficiently, yet their substantial memory footprint poses a major barrier to deployment. Expert pruning reduces this footprint, but most methods focus on which experts to remove and prune every layer equally. We propose \textbf{SPARE} (\underline{\textbf{Sp}}arsity \underline{\textbf{A}}llocation via \underline{\textbf{R}}outing \underline{\textbf{E}}ntropy), a training-free, selector-agnostic method that decides how many experts to remove from each layer using routing entropy, the entropy of tokens' normalized Top-$K$ gate weights on calibration data. When a token's gate weight is spread evenly, removing one active expert discards little of it, so SPARE prunes layers with more diffuse routing more aggressively. A perturbation bound under random pruning formalizes this intuition, and controlled interventions with a fixed selector link concentrated routing to larger output changes. With the same selector and budget, SPARE improves average accuracy over uniform allocation on Mixtral, Qwen3, and DeepSeek models, especially at high pruning ratios. At 50\% expert pruning, it raises the eight-task average of Qwen3-30B-A3B from 57.2\% to 63.8\%, and the math, code, and science average of DeepSeek-R1 from 80.1 to 82.8.
\end{abstract}

\section{Introduction}
\label{intro}
The pursuit of stronger reasoning capabilities has pushed Large Language Models (LLMs) to the trillion-parameter scale, incurring prohibitive computational costs. The Mixture-of-Experts (MoEs) architecture addresses this challenge through \textit{conditional computation}, which activates only a sparse subset of parameters per input token~\citep{shazeer2017outrageously, lepikhin2020gshard, fedus2022switch, dai2024deepseekmoe}. While conditional computation drastically reduces FLOPs, the memory overhead necessitates loading the full set of experts. Furthermore, the requirement to distribute these experts across devices complicates the communication topology, leading to efficiency bottlenecks. This highlights the urgent need for compression techniques capable of shrinking pre-trained MoEs.

To address this, recent research has focused on expert pruning, which removes entire experts to reduce model size. Existing approaches usually estimate an expert-level importance score from activation frequency, routing statistics, or reconstruction error~\citep{dong2025domain,lasby2025reap}.
These scores provide useful criteria for selecting experts within a layer. However, using expert importance directly for cross-layer pruning is less reliable. Scores computed in different layers are not necessarily calibrated, since activation scales, expert functions, and token distributions vary across depth. The calibration domain can further change which experts appear important. As a result, globally ranking experts across all layers may over-prune sensitive layers~\citep{chowdhury2024provably,lasby2025reap}.

This limitation comes from conflating two decisions: \emph{inter-layer allocation}, which determines how many experts to remove from each layer under a global budget, and \emph{intra-layer selection}, which determines which experts to remove within a layer. Existing criteria such as REAP and EASY-EP address the second; SPARE addresses the first and can be paired with any of them. 
Non-uniform allocation has been explored in dense LLMs~\citep{men2025shortgpt,yin2024outlier} and, more recently, in MoE expert pruning~\citep{yang2024moe,li2026cgnnprune}, using signals such as calibration loss or expert-usage frequency. SPARE instead allocates budgets by how gate weight is distributed among each token's active experts, a training-free signal distinct from usage statistics across inputs.

To capture these routing dynamics, we use \emph{routing entropy}: the Shannon entropy of each token's normalized Top-$K$ gate weights, averaged over calibration tokens. As illustrated in Figure~\ref{fig:teaser_entropy_sensitivity}, high-entropy layers spread gate weight broadly among the active experts, so removing one active expert discards only a small share of each token's weight, which is then reassigned to the remaining experts. Low-entropy layers concentrate weight on fewer experts, suggesting stronger input-conditioned specialization and making aggressive pruning riskier. This signal is layer-level rather than expert-level: it assumes neither that particular experts are indispensable nor that the same experts are selected across domains. Although expert identities shift between web text, code, math, and science, layer-wise entropy rankings remain stable across calibration domains, making them a robust basis for allocation.

Leveraging this insight, we introduce \textbf{SPARE}, a training-free sparsity allocation framework
that assigns non-uniform pruning budgets across MoE layers. Our main contributions are:
\begin{itemize}[leftmargin=*, itemsep=1pt, topsep=2pt, parsep=0pt, partopsep=0pt]
    \item \textbf{Entropy-guided, plug-and-play allocation.}
    We formulate MoE expert pruning as an inter-layer budget allocation problem and replace uniform allocation with entropy-aware budgets. SPARE decides how many experts to prune in each layer, while existing criteria such as REAP and EASY-EP decide which ones, without retraining or modifying the router.

    \item \textbf{Theoretical foundation.}
    We derive a perturbation bound that, under uniform random pruning, grows with the routing entropy gap. The analysis provides a layer-level sensitivity signal and an existence argument for deterministic masks.
    
    \item \textbf{Empirical validation.}
    With the same selector and budget, SPARE improves average accuracy over uniform allocation on Mixtral, Qwen3, and DeepSeek models, especially at high pruning ratios. At 50\% expert pruning, it raises the eight-task average of Qwen3-30B-A3B from 57.2\% to 63.8\%, and the math, code, and science average of DeepSeek-R1 from 80.1 to 82.8.
\end{itemize}

\section{Related Work}
\label{sec:related}
\paragraph{Routing Heterogeneity.}
MoEs scale model capacity through conditional computation~\citep{shazeer2017outrageously,fedus2022switch}.
Despite load-balancing objectives, expert behavior can remain non-uniform: \citet{su2025unveiling} identify ``super-experts'' that produce extreme activation outliers, and \citet{lo2024closer} examine relationships between routing weights and expert outputs. These observations motivate studying routing heterogeneity across layers.
\paragraph{Expert Pruning and Compression.}
Expert pruning reduces the memory footprint of pretrained MoEs. Selection criteria use routing statistics~\citep{dong2025domain,koishekenov2023memory}, router-weighted activation saliency~\citep{lasby2025reap}, or reconstruction loss~\citep{lu2024not}; EEP~\citep{liu2024efficient} instead uses evolutionary search.
Shapley-MoE~\citep{huang2025discovering} ranks experts globally across layers via router-guided sampling; we exclude it from comparisons under fixed layer budgets, which would alter its selection rule.
\citet{chowdhury2024provably} provide a pruning guarantee for fine-tuned sparse MoEs based on changes in router distributions. SEER-MoE~\citep{muzio2024seer}, MoE-SVD~\citep{li2025moesvd}, and AdaWhiten~\citep{adawhiten2026} explore further pruning and decomposition strategies, alongside task-specific pruning~\citep{chen2022taskspecific} and benchmarks such as MC-Suite~\citep{jaiswal2025finding}. Expert merging, skipping, dynamic Top-$K$ routing, and intra-expert pruning~\citep{xie2024moe} operate at other granularities; SPARE instead targets budget allocation.

\paragraph{Layer-wise Sparsity Allocation.}
In dense LLMs, ShortGPT~\citep{men2025shortgpt} studies
block removal, whereas OWL~\citep{yin2024outlier} allocates weight sparsity using activation-aware
outlier statistics. SlimGPT~\citep{ling2024slimgpt} also adopts non-uniform pruning schedules.
For MoEs, MoE-I$^2$~\citep{yang2024moe} estimates layer importance from calibration losses after
removing individual experts, assigns non-uniform pruning ratios, and selects experts through
genetic and block-wise search; its full pipeline additionally includes low-rank decomposition and LoRA recovery.
C-GNN-PRUNE~\citep{li2026cgnnprune} computes the entropy of expert-selection frequencies pooled over calibration samples and retains more experts in layers with higher usage entropy. Its expert-selection procedure uses learned graph embeddings. SPARE averages, over tokens, the entropy of each token's normalized Top-$K$ gate weights and prunes more experts in layers with higher gate entropy. Usage entropy describes how broadly expert identities are used across inputs; gate entropy describes how evenly weight is spread among the experts active for a single token. The rules therefore treat high entropy in opposite ways without contradiction: they measure different distributions. SPARE contributes this routing signal and a selector-independent allocation rule.

\begin{figure*}[t]
  \centering
  \includegraphics[width=\textwidth]{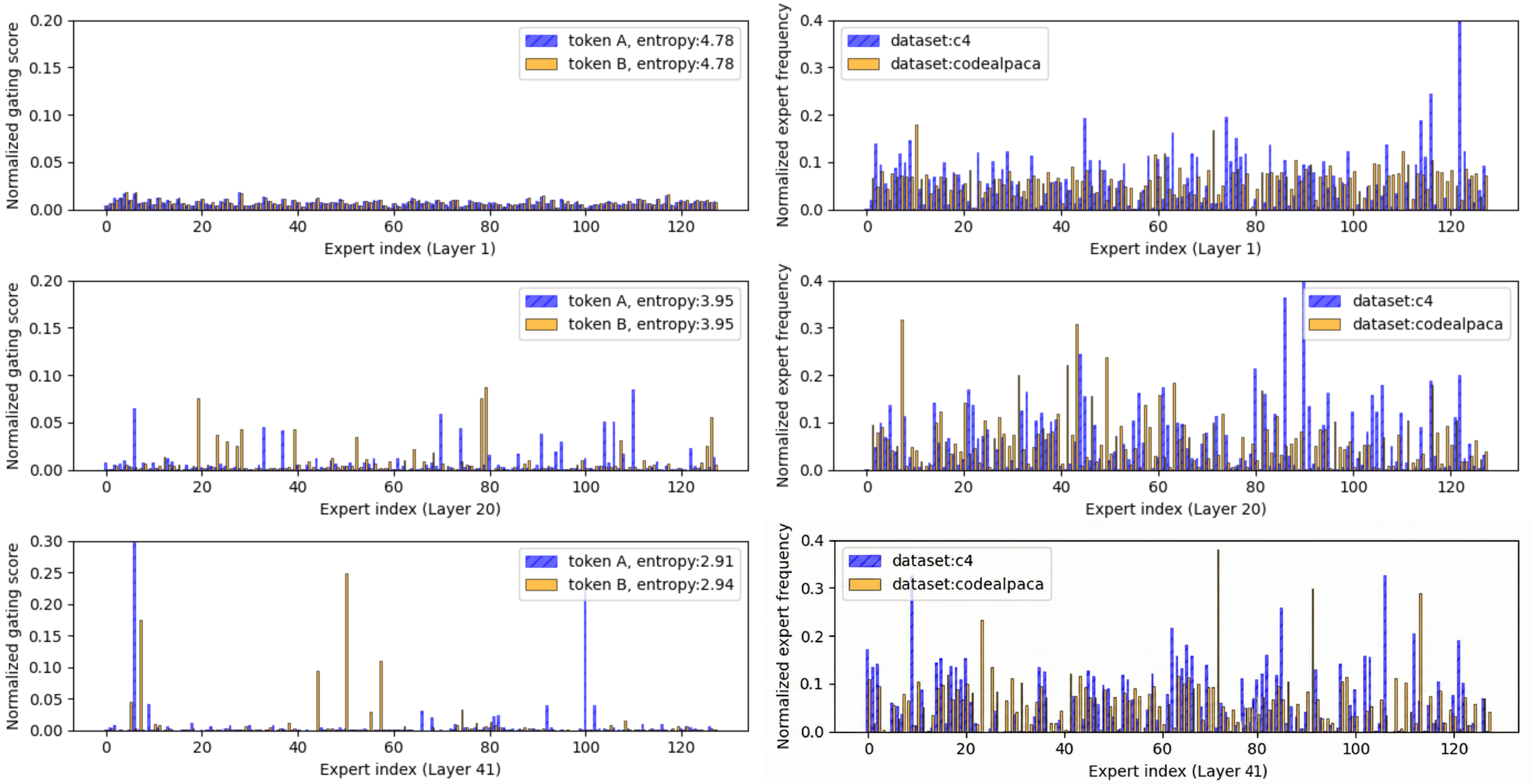}
  \vspace{-2mm}
\caption{
Visualization of routing heterogeneity and domain sensitivity on Qwen3-30B-A3B.
Left: gating distributions for two representative tokens at selected layers.
These single-token examples illustrate diffuse and concentrated routing; quantitative analyses use layer-aggregated statistics. Right: expert selection frequencies on C4 and CodeAlpaca. The selected expert
identities can differ substantially across domains, which makes pruning directly from single-domain
expert frequencies brittle. SPARE instead uses layer-wise routing entropy, an aggregated
concentration statistic, to guide inter-layer sparsity allocation.
}
\label{fig:teaser_entropy_sensitivity}
\vspace{-2mm}
\end{figure*}

\section{Routing Entropy and Heterogeneity of MoE Layers}
\label{sec:prelim}

We analyze layer-wise routing heterogeneity in MoEs and define routing entropy as the signal used
by SPARE for inter-layer budget allocation.

\subsection{Problem Formulation}
\label{sec:prelim:setup}

Consider an MoE layer $\ell$ with $N$ experts. For token $\mathbf{x}$, the router selects the Top-$K$
expert set $\mathcal{T}_\ell(\mathbf{x})$ and computes the normalized gating distribution
\begin{equation}
    g_{\ell,j}(\mathbf{x}) =
    \frac{\exp(z_{\ell,j}(\mathbf{x}))}{\sum_{t\in\mathcal{T}_\ell(\mathbf{x})}\exp(z_{\ell,t}(\mathbf{x}))}
    \mathbf{1}\{j\in\mathcal{T}_\ell(\mathbf{x})\}.
    \label{eq:topk-gate}
\end{equation}
Let $\mathbf{m}=[m_0,\ldots,m_{L-1}]^\top$ denote the layer-wise pruning profile, where $m_\ell$
is the number of experts removed from layer $\ell$. Given a global target $M_{\mathrm{target}}$,
SPARE solves the inter-layer allocation problem
\begin{equation}
    \sum_{\ell=0}^{L-1} m_\ell=M_{\mathrm{target}}, \qquad 0\le m_\ell\le N-K,
    \label{eq:budget_constraint}
\end{equation}
while leaving the intra-layer choice of removed experts to an expert-ranking rule such as
EASY-EP~\citep{dong2025domain} or REAP~\citep{lasby2025reap}.
\subsection{Routing Entropy Gap}
\label{sec:prelim:specialization}

For each token, we define the Top-$K$ routing entropy as
\begin{equation}
    H_\ell(\mathbf{x})
    =
    -\sum_{j\in\mathcal{T}_\ell(\mathbf{x})}
    g_{\ell,j}(\mathbf{x})\log g_{\ell,j}(\mathbf{x}).
    \label{eq:token_entropy}
\end{equation}
The layer-average entropy and routing entropy gap are
\begin{equation}
    \bar{H}_\ell
    =
    \mathbb{E}_{\mathbf{x}\sim\mathcal{D}}[H_\ell(\mathbf{x})],
    \qquad
    S_\ell
    =
    \log K-\bar{H}_\ell .
    \label{eq:specialization_def}
\end{equation}
Equivalently,
\begin{equation}
    S_\ell
    =
    \mathbb{E}_{\mathbf{x}\sim\mathcal{D}}
    \left[
    D_{\mathrm{KL}}
    \left(
    \mathbf{g}_\ell(\mathbf{x})
    \middle\|
    u_{\mathcal{T}_\ell(\mathbf{x})}
    \right)
    \right],
    \label{eq:specialization_kl}
\end{equation}
where $u_{\mathcal{T}_\ell(\mathbf{x})}$ is the uniform distribution over the active experts. A larger
$S_\ell$ means sharper routing concentration and motivates a more conservative pruning budget.
\subsection{Layer-wise Routing Heterogeneity}
\label{sec:motivation}

Figure~\ref{fig:teaser_entropy_sensitivity} shows that gating concentration varies across layers (left), while selected expert identities vary across domains (right). SPARE therefore allocates budgets with the layer-level statistic $\bar{H}_\ell$ rather than domain-specific expert frequencies. Routing entropy indicates how much gate weight a token loses when an active expert is removed, but not whether the remaining experts can compensate, so we test pruning tolerance directly.

We divide the 48 layers of Qwen3-30B-A3B into entropy terciles and use REAP to prune 50\% of experts within one tercile at a time, removing about 16.7\% of all routed experts. From an unpruned eight-task average of 72.1\%, pruning the high-, medium-, and low-entropy terciles yields 68.2\%, 67.2\%, and 66.4\%, respectively, supporting entropy as a pruning-tolerance signal in this setting (see also Sections~\ref{sec:exp:risk} and~\ref{sec:calibration}).

\section{Theoretical Analysis: Routing-Entropy Perturbation Bounds}
\label{sec:theory}
We analyze how routing entropy can guide the allocation of a fixed expert-pruning budget across layers.
Under the stated routing and bounded-output assumptions, we first bound the output perturbation for any fixed pruning mask using the routing mass assigned to removed experts.
For uniform random pruning, this bound takes an explicit form involving the routing entropy gap.
We then give a sufficient condition under which a deterministic expert selector inherits the same entropy-dependent bound, with an additive term when the condition holds approximately.
Finally, we describe how local pruning perturbations propagate through the network.

\subsection{Single-Layer Pruning Perturbation}
Consider a layer $\ell$ with output $y_\ell(\mathbf{x}) = \sum_j g_{\ell,j}(\mathbf{x}) f_{\ell,j}(\mathbf{x})$. After removing $P\subseteq\{1,\ldots,N\}$, $|P|=m\le N-K$, we recompute Top-$K$ among remaining experts with unchanged logits and fixed tie-breaking, then apply softmax on the new support to obtain $g_\ell^{(P)}$. Write $y^{(P)}_\ell(\mathbf{x})=\sum_j g^{(P)}_{\ell,j}(\mathbf{x})f_{\ell,j}(\mathbf{x})$ and $\delta_\ell(\mathbf{x},P)=y^{(P)}_\ell(\mathbf{x})-y_\ell(\mathbf{x})$.

The following lemma quantifies the routing change.

\begin{lemma}[Exact $\ell_1$ Identity under Top-$K$ Rerouting]
\label{lemma:exact_identity}
Let $\mathcal{E}_\ell(\mathbf{x}, P) \triangleq \sum_{j \in \mathcal{T}_\ell(\mathbf{x}) \cap P} g_{\ell,j}(\mathbf{x})$ be the probability mass assigned to pruned experts. The $\ell_1$ distance between the original and pruned gating distributions satisfies:
\begin{equation}
    \big\|g_\ell^{(P)}(\mathbf{x}) - g_\ell(\mathbf{x})\big\|_1 \;=\; 2\,\mathcal{E}_\ell(\mathbf{x}, P).
\end{equation}
\end{lemma}
\textit{Proof.} The new softmax normalizer cannot increase, so surviving original Top-$K$ weights cannot decrease. Unpruned experts gain exactly the removed mass $\mathcal{E}_\ell$, yielding an $\ell_1$ change of $2\mathcal{E}_\ell$ (Appendix~\ref{appendix:single_layer}).

% Using this identity, we connect the expected perturbation error to the layer's specialization. This theorem serves as the theoretical bridge between our metric $S_\ell$ and the layer sensitivity.

% \begin{theorem}[Expected Perturbation and Specialization]
% \label{thm:single_layer}
% Assume bounded expert outputs $\|f_{\ell,j}(\mathbf{x})\|_1 \le B$. Under uniform random pruning of $m$ experts, the expected squared output perturbation is bounded by:
% \begin{equation}
% \label{eq:single_layer_bound}
%     \mathbb{E}_{\mathbf{x}, P}\big[\|\delta_\ell(\mathbf{x}, P)\|_1^2\big] 
%     \;\le\; \underbrace{C(N, m, B)}_{\text{Scaling Factor}} \cdot \left( \frac{1}{K} + 2 \cdot S_\ell \right),
% \end{equation}
% where $S_\ell = \log K - \bar{H}_\ell$ is the routing entropy gap.
% \end{theorem}
% \textit{Proof Sketch.} We bound $\|\delta_\ell\|_1$ using Lemma~\ref{lemma:exact_identity} to show that error scales with the second moment of the gating weights $\sum g_j^2$. We then invoke the relationship between the collision entropy (related to $\sum g^2$) and the Shannon entropy $\bar{H}_\ell$, demonstrating that higher deviation from uniformity (larger $S_\ell$) strictly increases the expected error bound (Full proof shown in Appendix~\ref{appendix:single_layer}).
We first separate the mask-wise perturbation bound from the distribution used to select pruning masks.

\begin{proposition}[Perturbation bound for arbitrary selectors]
\label{prop:general_selector}
Fix $0\le m\le N-K$ and let $\mathcal{P}_{\ell,m}$ contain all size-$m$ pruning masks. Under the rerouting rule in Appendix~\ref{app:pruning_renorm} and bounded expert outputs $\|f_{\ell,j}(x)\|_1\le B_\ell$, every admissible $P$ and $x$ satisfy
\begin{equation}
\|\delta_\ell(x,P)\|_1^2
\le 4B_\ell^2\mathcal{E}_\ell(x,P)^2.
\label{eq:maskwise_perturbation}
\end{equation}
Condition on fixed calibration data. For any token-independent distribution $q_\ell$ on $\mathcal{P}_{\ell,m}$, let $\mathcal{M}_\ell(q_\ell)=\mathbb{E}_{x,P\sim q_\ell}[\mathcal{E}_\ell(x,P)^2]$. Then
\begin{equation}
\mathbb{E}_{x,P\sim q_\ell}[\|\delta_\ell(x,P)\|_1^2]
\le 4B_\ell^2\mathcal{M}_\ell(q_\ell).
\label{eq:general_selector_bound}
\end{equation}
A deterministic selector has $q_\ell(P)=\mathbf{1}\{P=P_{\mathcal A}\}$.
\end{proposition}

Uniform random pruning gives a closed-form special case of Proposition~\ref{prop:general_selector}.

\begin{theorem}[Expected perturbation under random pruning]
\label{thm:single_layer_entropy}
Under Proposition~\ref{prop:general_selector}'s assumptions, let $P\sim\mathrm{Unif}(\mathcal{P}_{\ell,m})$ independently of $x$. Then
\begin{equation}
\begin{aligned}
\mathbb{E}_{x,P}\!\left[\|\delta_\ell(x,P)\|_1^2\right]
&\le 4B_\ell^2 \Bigl(
c_0(m,N) \\
&\quad + c_1(m,N)
\left[\frac{1}{K}+2S_\ell\right]
\Bigr),
\end{aligned}
\label{eq:entropy_bound_exact}
\end{equation}
where
\[
\begin{aligned}
c_0(m,N)&=\frac{m(m-1)}{N(N-1)},\\
c_1(m,N)&=\frac{m(N-m)}{N(N-1)}.
\end{aligned}
\]
and $S_\ell=\log K-\bar{H}_\ell$ is the routing entropy gap.
\end{theorem}

The proof combines the uniform-subset second moment with Pinsker's inequality; see Appendix~\ref{appendix:single_layer}.

\begin{corollary}[Entropy bound for deterministic selectors]
\label{cor:selector_transfer}
Under the same assumptions, write $\Psi_\ell(m)$ for the bracketed term in Eq.~\eqref{eq:entropy_bound_exact} and let $U_{\ell,m}=\mathrm{Unif}(\mathcal{P}_{\ell,m})$. If a deterministic selector returns $P_{\mathcal A}$ satisfying
\begin{equation}
\begin{aligned}
\mathbb{E}_x[\mathcal{E}_\ell(x,P_{\mathcal A})^2]
\le{}&\mathcal{M}_\ell(U_{\ell,m})\\
&+\varepsilon_{\mathcal A,\ell},
\end{aligned}
\label{eq:selector_dominance}
\end{equation}
for some $\varepsilon_{\mathcal A,\ell}\ge0$. Then
\begin{equation}
\begin{aligned}
\mathbb{E}_x[\|\delta_\ell(x,P_{\mathcal A})\|_1^2]
\le{}&4B_\ell^2\bigl(\Psi_\ell(m)\\
&+\varepsilon_{\mathcal A,\ell}\bigr).
\end{aligned}
\label{eq:deterministic_selector_bound}
\end{equation}
If this second moment is no greater than the uniform random pruning baseline, $\varepsilon_{\mathcal A,\ell}=0$ suffices.
\end{corollary}

Proofs of Proposition~\ref{prop:general_selector} and Corollary~\ref{cor:selector_transfer} are in Appendix~\ref{app:general_selector_proofs}. The condition in Eq.~\eqref{eq:selector_dominance} is not established for REAP or EASY-EP.

\subsection{End-to-End Propagation as a Weighting Factor}

Consider an $L$-layer residual network with original states $h_\ell$, pruned states $\hat h_\ell$, and identical initial states. Assume deterministic Lipschitz constants $\lambda_\ell$ hold uniformly over admissible masks on a common domain containing both trajectories. Unrolling the error recursion gives:

\begin{proposition}[Propagation decomposition]
\label{prop:propagation}
For any fixed sequence of pruning masks $\{P_\ell\}_{\ell=0}^{L-1}$,
\begin{equation}
\begin{aligned}
\|\hat{h}_L-h_L\|_1
&\le \sum_{\ell=0}^{L-1}
\Gamma_\ell
\|\delta_\ell(h_\ell,P_\ell)\|_1,\\
\Gamma_\ell
&=\alpha\prod_{k=\ell+1}^{L-1}(1+\alpha\lambda_k).
\end{aligned}
\label{eq:propagation_decomposition}
\end{equation}
\end{proposition}

The proof follows by discrete Gronwall unrolling; see Appendix~\ref{appendix:e2e_bound}. Taking expectations and applying Proposition~\ref{prop:general_selector} yields
\begin{equation}
\mathbb{E}\|\hat h_L-h_L\|_1
\le 2\sum_{\ell=0}^{L-1}
\Gamma_\ell B_\ell\sqrt{\mathcal{M}_\ell(q_\ell)}.
\label{eq:e2e_general_selector}
\end{equation}
Here $\mathcal{M}_\ell$ uses the original layer-input distribution, with masks independent of the evaluation input given fixed calibration data. Its bound is $\Psi_\ell(m_\ell)$ for uniform random pruning and $\Psi_\ell(m_\ell)+\varepsilon_{\mathcal A,\ell}$ under Eq.~\eqref{eq:selector_dominance}. The decomposition motivates a propagation regularizer, without identifying an exact estimator of $\Gamma_\ell$.

\subsection{Implications for Pruning Strategy}

For fixed $(m,N,K,B_\ell)$ and selector residual, the bound increases with $S_\ell=\log K-\bar H_\ell$. SPARE does not minimize the summed bound; it uses $S_\ell$ as a layer-ranking signal and ignores the layer-dependent scale $B_\ell$. The allocation score in Section~\ref{sec:method} is therefore a heuristic.

Since exact Lipschitz products are intractable for large MoEs, SPARE uses a cosine-based proxy as an empirical propagation regularizer; entropy remains the primary allocation signal.

\section{Entropy-Aware Sparsity Allocation}
\label{sec:method}

SPARE converts layer-wise routing entropy into an integer pruning profile under the global budget.
It uses entropy as the primary signal and an optional propagation-aware term as a lightweight
regularizer.

\subsection{Entropy-Based Allocation Score}
\label{sec:method:redundancy}

Since $S_\ell=\log K-\bar{H}_\ell$, smaller entropy gap is equivalent to larger routing entropy.
We therefore define the normalized allocation score
\begin{equation}
    \mathcal{R}_\ell
    =
    \frac{\bar{H}_\ell-\min_k\bar{H}_k}
    {\max_k\bar{H}_k-\min_k\bar{H}_k}.
    \label{eq:redundancy_score}
\end{equation}
Layers with larger $\mathcal{R}_\ell$ are assigned larger pruning budgets. A contrast coefficient
$p$ controls the sharpness of allocation:
\begin{equation}
    \mathcal{W}_{\mathrm{base}}^{(\ell)}
    =
    (\mathcal{R}_\ell)^p .
    \label{eq:base_weight}
\end{equation}

\subsection{Propagation-Aware Regularization}
\label{sec:method:stability}

Equation~\eqref{eq:propagation_decomposition} motivates accounting for downstream propagation,
but exact Lipschitz products are intractable for large MoEs. We use cosine similarity only as a
data-dependent proxy. Let $c_k$ be the average cosine similarity between the input and output
activations of layer $k$. We define
\begin{equation}
    D_\ell
    =
    \sum_{k=\ell+1}^{L-1}(1-c_k),
    \qquad
    \phi_{\mathrm{stab}}^{(\ell)}
    =
    \exp(-\beta D_\ell).
    \label{eq:stability_weight}
\end{equation}
This term mildly reduces the pruning weight of layers whose perturbations pass through more
transformative downstream blocks. When $\beta=0$, SPARE reduces to entropy-only allocation.

\subsection{Discrete Budget Projection}
\label{sec:method:optimization}

The final continuous allocation weight is
\begin{equation}
    \mathcal{W}_\ell
    =
    \mathcal{W}_{\mathrm{base}}^{(\ell)}
    \phi_{\mathrm{stab}}^{(\ell)} .
    \label{eq:final_weight}
\end{equation}
We convert $\{\mathcal{W}_\ell\}$ into an integer pruning profile by solving for a global scaling
factor $\lambda^\ast$:
\begin{equation}
    \sum_{\ell=0}^{L-1}
    \mathcal{P}_{\mathbb{Z}}
    \left(
    \lambda^\ast \mathcal{W}_\ell
    \right)
    =
    M_{\mathrm{target}},
    \label{eq:integer_projection}
\end{equation}
where $\mathcal{P}_{\mathbb{Z}}(\cdot)$ rounds to the nearest feasible integer under
$0\le m_\ell\le N-K$. Since the left-hand side is monotone in $\lambda^\ast$, we solve it by
bisection and obtain the final profile $\mathbf{m}$.

\section{Experiments}
\label{sec:experiments}

We ask three questions: \textbf{(RQ1)} Does entropy-guided allocation improve benchmark performance over uniform allocation at the same pruning budget? \textbf{(RQ2)} Does the gain persist across the evaluated expert selectors, and does entropy correlate with pruning risk when the selector is fixed? \textbf{(RQ3)} How stable are entropy profiles across calibration domains, and how does non-uniform allocation affect throughput in the tested implementation?
\subsection{Experimental Setup}

\paragraph{Models.}
We evaluate \texttt{Mixtral-8x7B-\allowbreak Instruct-v0.1}, \texttt{Qwen3-30B-A3B},
\texttt{Qwen3-235B-A22B}, and \texttt{DeepSeek-R1}. 

\paragraph{Benchmarks.} We assess capabilities across four distinct domains: 
\textbf{(i) Multiple-choice tasks:} MMLU~\citep{hendrycks2020measuring}, ARC-Challenge (ARC-c) / ARC-Easy (ARC-e)~\citep{clark2018think}, BoolQ~\citep{clark2019boolq}, HellaSwag (Hella)~\citep{zellers2019hellaswag}, OpenBookQA (OBQA)~\citep{mihaylov2018can}, Recognizing Textual Entailment Challenge (RTE)~\citep{bentivogli2009fifth}, and WinoGrande (WinG)~\citep{sakaguchi2021winogrande}; 
\textbf{(ii) Reasoning:} MATH-500~\citep{lightman2024lets}; 
\textbf{(iii) Science question answering:} GPQA-Diamond~\citep{rein2023gpqa}; and 
\textbf{(iv) Code:} LiveCodeBench~\citep{jain2025livecodebench}.

\paragraph{Calibration.}
For Mixtral and Qwen3-30B-A3B, we use a calibration set of 1024 randomly sampled sequences from C4 to estimate layer-average routing entropy ($\bar{H}_\ell$).
For large-scale reasoning models such as DeepSeek-R1, we sample 25 instances each from AIME (math), LiveCodeBench (code), and GPQA (science).

\paragraph{Baselines and Protocols.}
To disentangle the benefits of \emph{sparsity allocation} from \emph{expert selection}, we structure our comparison along two orthogonal dimensions.
\textbf{1) Intra-layer Expert Ranking:} Determines expert priority within a layer. We employ: (i) \textit{Random}; (ii) \textit{EASY-EP}~\citep{dong2025domain}; and (iii) \textit{REAP}~\citep{lasby2025reap}.
\textbf{2) Inter-layer Sparsity Allocation:} Determines the pruning budget $m_\ell$ per layer. We compare: (i) \textbf{Uniform}; (ii) \textbf{OWL}~\citep{yin2024outlier}; and (iii) \textbf{SPARE}.

% \subsection{Experimental Setup}

% \paragraph{Models.}
% To ensure architectural diversity and scalability, we select representative models from three state-of-the-art MoE families: \textbf{(1) Mixtral:} \texttt{Mixtral-8x7B-Instructv0.1}; \textbf{(2) Qwen:} \texttt{Qwen3-30B-A3B} and \texttt{Qwen3-235B-A22B}; \textbf{(3) DeepSeek:} \texttt{DeepSeek-R1 (671B)}.

\begin{table*}[t]
    \centering
    \caption{Benchmark accuracy (\%) and unweighted eight-task averages under expert pruning. \textbf{PR} denotes the global pruning ratio. \textbf{Inter.}: uniform (Uni.) or \ours budget allocation. \textbf{Intra.}: Random, REAP, or EASY-EP expert selection. \textbf{Bold} indicates the best score within each model/PR setting; gains in \textbf{Avg.} compare \ours with uniform allocation using the same selector.}
    \label{tab:full_results}

    \begingroup
    \scriptsize
    \setlength{\tabcolsep}{2.3pt}
    \renewcommand{\arraystretch}{0.90}
    \resizebox{\textwidth}{!}{%
    \begin{tabular}{p{2.10cm}|c|cc|ccccccccc}
        \toprule
        \textbf{Model} & \textbf{PR} & \textbf{Inter.} & \textbf{Intra.}
        & \textbf{ARC-c} & \textbf{ARC-e} & \textbf{BoolQ} & \textbf{Hella} & \textbf{MMLU} & \textbf{OBQA} & \textbf{RTE} & \textbf{WinG} & \textbf{Avg.$\uparrow$} \\
        \midrule

        \multirow{9}{*}{\shortstack[l]{\textbf{Mixtral-8x7B}\\\textbf{-Instruct-v0.1}}}
        & 0\%  & \multicolumn{2}{c|}{Baseline}
              & 65.0 & 84.2 & 88.7 & 86.1 & 69.1 & 49.6 & 72.2 & 74.0 & 73.6 \\
        \cmidrule{2-13}

        & \multirow{4}{*}{25\%}
          & \multirow{2}{*}{Uni.} & Random
              & 59.6 & 81.4 & 86.9 & 82.9 & 62.7 & 46.8 & 71.1 & 75.6 & 70.9 \\
        &   &                  & REAP
              & 60.6 & 80.7 & 87.5 & 83.5 & 63.3 & 48.6 & \textbf{71.9} & 70.9 & 70.9 \\
        \cmidrule(lr){3-13}
        &   & \multirow{2}{*}{\ours} & Random
              & 60.1 & \textbf{81.9} & 87.1 & 83.2 & 63.1 & 47.5 & \textbf{71.9} & \textbf{75.8} & 71.3~($\uparrow$0.4) \\
        &   &                      & REAP
              & \textbf{61.8} & 81.5 & \textbf{87.9} & \textbf{84.0} & \textbf{64.5} & \textbf{49.2} & \textbf{71.9} & 72.1 & \textbf{71.6}~($\uparrow$\textbf{0.7}) \\
        \cmidrule{2-13}

        & \multirow{4}{*}{50\%}
          & \multirow{2}{*}{Uni.} & Random
              & 51.4 & 73.3 & 81.5 & 75.7 & 52.1 & 41.0 & 70.2 & 69.5 & 64.3 \\
        &   &                  & REAP
              & 50.9 & 76.2 & 85.3 & 78.1 & 51.9 & \textbf{44.6} & \textbf{71.6} & 72.5 & 66.4 \\
        \cmidrule(lr){3-13}
        &   & \multirow{2}{*}{\ours} & Random
              & 50.9 & 75.5 & 81.7 & 77.2 & \textbf{54.4} & 42.0 & 70.8 & 70.2 & 65.3~($\uparrow$1.0) \\
        &   &                      & REAP
              & \textbf{53.6} & \textbf{78.4} & \textbf{86.9} & \textbf{78.2} & 53.1 & 44.4 & 70.4 & \textbf{74.1} & \textbf{67.4}~($\uparrow$\textbf{1.0}) \\
        \midrule

        \multirow{13}{*}{\shortstack[l]{\textbf{Qwen3-30B}\\\textbf{-A3B}}}
        & 0\%  & \multicolumn{2}{c|}{Baseline}
              & 56.2 & 79.1 & 88.3 & 77.6 & 77.9 & 44.6 & 81.6 & 71.1 & 72.1 \\
        \cmidrule{2-13}

        & \multirow{6}{*}{25\%}
          & \multirow{3}{*}{Uni.} & Random
              & 50.8 & 75.6 & 84.8 & 70.7 & 69.7 & 40.2 & 80.5 & 68.8 & 67.6 \\
        &   &                  & REAP
              & 55.8 & 78.3 & 87.1 & 75.8 & 71.9 & 44.6 & 79.4 & 70.5 & 70.4 \\
        &   &                  & Easy-EP
              & 53.2 & 78.0 & 88.8 & 76.6 & 74.2 & 44.0 & \textbf{81.9} & \textbf{70.7} & 70.9 \\
        \cmidrule(lr){3-13}
        &   & \multirow{3}{*}{\ours} & Random
              & 50.6 & 76.8 & 85.5 & 70.8 & 68.8 & 41.2 & 79.0 & 69.0 & 67.7~($\uparrow$0.1) \\
        &   &                      & REAP
              & \textbf{56.0} & \textbf{78.4} & 87.5 & 76.0 & 72.1 & \textbf{44.8} & 80.1 & 70.6 & 70.7~($\uparrow$0.3) \\
        &   &                      & Easy-EP
              & 53.8 & \textbf{78.4} & \textbf{89.0} & \textbf{76.9} & \textbf{74.4} & 44.6 & 81.5 & 69.4 & \textbf{71.0}~($\uparrow$\textbf{0.1}) \\
        \cmidrule{2-13}

        & \multirow{6}{*}{50\%}
          & \multirow{3}{*}{Uni.} & Random
              & 40.0 & 62.8 & 77.9 & 55.5 & 49.6 & 25.4 & 74.1 & 61.1 & 55.8 \\
        &   &                  & REAP
              & 40.5 & 61.2 & 82.9 & 64.4 & 51.1 & 36.2 & 61.0 & 64.5 & 57.7 \\
        &   &                  & Easy-EP
              & 37.2 & 58.1 & 81.0 & 66.7 & 49.0 & 37.6 & 63.2 & 64.7 & 57.2 \\
        \cmidrule(lr){3-13}
        &   & \multirow{3}{*}{\ours} & Random
              & 39.8 & 63.1 & 77.1 & 55.9 & 50.6 & 31.8 & 74.0 & 61.2 & 56.7~($\uparrow$0.9) \\
        &   &                      & REAP
              & \textbf{43.1} & 64.6 & 83.3 & 64.5 & 48.9 & 37.0 & 74.7 & 63.6 & 60.0~($\uparrow$2.3) \\
        &   &                      & Easy-EP
              & 42.5 & \textbf{65.2} & \textbf{84.9} & \textbf{73.2} & \textbf{59.1} & \textbf{41.2} & \textbf{75.8} & \textbf{68.4} & \textbf{63.8}~($\uparrow$\textbf{6.6}) \\
        \bottomrule
    \end{tabular}}
    \endgroup
\end{table*}

\subsection{Main Results}
\label{sec:exp:main_small}
Table~\ref{tab:full_results} shows improvements in macro-average accuracy in all ten evaluated combinations of model, pruning ratio, and intra-layer selector. The largest gain occurs on Qwen3-30B-A3B at 50\% pruning: with EASY-EP fixed, \ours raises mean accuracy from 57.2\% to 63.8\%, a gain of 6.6 percentage points. At the same pruning ratio, gains with Random and REAP are 0.9 and 2.3 points, respectively, showing that the benefit depends on the selector. Improvements at 25\% pruning are smaller, ranging from 0.1 to 0.7 points.

% \begin{table}[!htbp]
% \centering
% \begingroup
% \scriptsize
% \setlength{\tabcolsep}{2.3pt}
% \renewcommand{\arraystretch}{0.90}
% \resizebox{\columnwidth}{!}{%
% \begin{tabular}{p{1.15cm}|c|c|ccc|c}
% \toprule
% \textbf{Model} & \textbf{PR} & \textbf{Inter.} & \textbf{LCB} & \textbf{Math} & \textbf{GPQA} & \textbf{Avg.$\uparrow$} \\
% \midrule

% \multirow{3}{*}{DS-R1}
%   & --  & Baseline & 84.5 & 96.0 & 69.7 & 83.4 \\
% \cmidrule(lr){2-7}
%   & \multirow{2}{*}{50\%} & Uni.  & 79.0 & 94.7 & 66.7 & 80.1 \\
%   &                       & \ours & 82.0 & 94.3 & 72.2 & 82.8~($\uparrow$2.7) \\
% \midrule

% \multirow{3}{*}{Qwen3}
%   & --  & Baseline & 92.5 & 96.0 & 77.5 & 88.7 \\
% \cmidrule(lr){2-7}
%   & \multirow{2}{*}{50\%} & Uni.  & 89.0 & 95.7 & 74.4 & 86.4 \\
%   &                       & \ours & 93.0 & 95.7 & 75.1 & 87.9~($\uparrow$1.5) \\
% \bottomrule
% \end{tabular}}
% \caption{\textbf{Scalability to massive reasoning models} under global 50\% expert pruning. \textbf{PR} denotes the global pruning ratio. \textbf{Inter.} denotes the inter-layer sparsity allocation strategy (Uniform vs. \ours). \textbf{Avg.} is computed over the three reported benchmarks (LCB, Math, GPQA), and the gain of \ours over Uniform is shown in parentheses. Abbreviations: DS-R1 = DeepSeek-R1 (671B); Qwen3 = Qwen3-235B-A22B; Uni. = Uniform.}
% \label{tab:large_scale_layerwise_compact}
% \endgroup
% \end{table}

\begin{table}[t]
\centering
\caption{Large-scale MoEs at 50\% global expert pruning (PR), using EASY-EP for both allocation rules. DS-R1: DeepSeek-R1; Qwen3: Qwen3-235B-A22B. LCB/Math/GPQA: LiveCodeBench/MATH-500/GPQA-Diamond. Avg. is the unweighted three-score mean; gains compare with uniform allocation.}
\label{tab:large_scale_layerwise_compact}
\begingroup
\scriptsize
\setlength{\tabcolsep}{3pt}
\renewcommand{\arraystretch}{0.88}
\resizebox{\columnwidth}{!}{%
\begin{tabular}{p{1.15cm}|c|c|ccc|c}
\toprule
\textbf{Model} & \textbf{PR} & \textbf{Inter.} & \textbf{LCB} & \textbf{Math} & \textbf{GPQA} & \textbf{Avg.$\uparrow$} \\
\midrule
\multirow{3}{*}{DS-R1}
  & --  & Baseline & 84.5 & 96.0 & 69.7 & 83.4 \\
\cmidrule(lr){2-7}
  & \multirow{2}{*}{50\%} & Uni.  & 79.0 & 94.7 & 66.7 & 80.1 \\
  &                       & \ours & 82.0 & 94.3 & 72.2 & 82.8~($\uparrow$2.7) \\
\midrule
\multirow{3}{*}{Qwen3}
  & --  & Baseline & 92.5 & 96.0 & 77.5 & 88.7 \\
\cmidrule(lr){2-7}
  & \multirow{2}{*}{50\%} & Uni.  & 89.0 & 95.7 & 74.4 & 86.4 \\
  &                       & \ours & 93.0 & 95.7 & 75.1 & 87.9~($\uparrow$1.5) \\
\bottomrule
\end{tabular}}
\endgroup
\vspace{-2mm}
\end{table}

\subsection{Controlled Analysis of Pruning Risk}
\label{sec:exp:risk}
\paragraph{Entropy groups with a fixed selector.}
On Mixtral-8$\times$7B, we estimate routing entropy using 128 C4 samples and divide layers into high-, middle-, and low-entropy groups. Within each group, we use the same selector, REAP, to prune four of the eight experts per layer while retaining the expert with the highest REAP score. Table~\ref{tab:mixtral_controlled_entropy_risk} reports normalized output perturbation on held-out samples. The low-entropy group has higher perturbation than the high-entropy group in all three domains, although the middle group does not follow a uniformly monotonic ordering. The point estimates indicate an association under this selection procedure.

\begin{table}[t]
\centering
\caption{Normalized output perturbation on Mixtral-8$\times$7B under the same REAP selection procedure. Columns denote layer entropy groups.}
\label{tab:mixtral_controlled_entropy_risk}
\begingroup
\small
\setlength{\tabcolsep}{4pt}
\renewcommand{\arraystretch}{1.05}
\begin{tabular}{lccc}
\toprule
Domain & High & Middle & Low \\
\midrule
WikiText2 & 0.721 & 0.641 & 0.726 \\
MATH-500 & 0.673 & 0.736 & 0.728 \\
HumanEval & 0.726 & 0.725 & 0.763 \\
\bottomrule
\end{tabular}
\endgroup
\end{table}

\paragraph{Single-layer pruning and output-distribution changes.}
We next test whether this association holds when each layer is pruned in
isolation. For each MoE layer, we remove its lowest-saliency experts under
REAP, with saliency estimated on C4, while keeping all other layers intact;
the pruning ratios in Table~\ref{tab:single_layer} apply only to the
intervened layer. We measure the resulting output-distribution change on C4,
Math, Code, and MMLU with KL divergence truncated to the unpruned model's
top-64 tokens, and summarize each layer by its mean over the four domains and
by its worst (maximum) domain. We define routing concentration as
$C_\ell = 1 - \bar{H}_\ell / \log K = S_\ell / \log K$, so a larger $C_\ell$
corresponds to lower routing entropy. In both models, at both pruning ratios,
and under both aggregations, the correlation between $C_\ell$ and
pruning-induced KL is positive, ranging from 0.26 to 0.45. These correlations
are moderate: routing concentration accounts for only part of the
layer-to-layer variation in pruning risk. We therefore use it to rank layers
for budget allocation, not to predict the perturbation of an individual layer.

\begin{table}[t]
\centering
\caption{Pearson  correlation between routing concentration
$C_\ell = 1 - \bar{H}_\ell/\log K$ and pruning-induced KL under single-layer
REAP pruning, computed over $n = 26$ (DeepSeek-V2-Lite) and $n = 48$
(Qwen3-30B-A3B) MoE layers. PR is the pruning ratio of the intervened layer;
KL is truncated to the teacher's top-64 tokens.}
\label{tab:single_layer}
\begingroup
\scriptsize
\setlength{\tabcolsep}{3.5pt}
\renewcommand{\arraystretch}{1.05}
\begin{tabular}{lccc}
\toprule
Model & PR & \shortstack{Mean-domain\\KL} & \shortstack{Worst-domain\\KL} \\
\midrule
DeepSeek-V2-Lite & 25\% & +0.435 & +0.450 \\
DeepSeek-V2-Lite & 50\% & +0.340 & +0.272 \\
Qwen3-30B-A3B & 25\% & +0.288 & +0.257 \\
Qwen3-30B-A3B & 50\% & +0.314 & +0.329 \\
\bottomrule
\end{tabular}
\endgroup
\end{table}

\begin{figure*}[t]
    \centering
    \includegraphics[width=.48\textwidth]{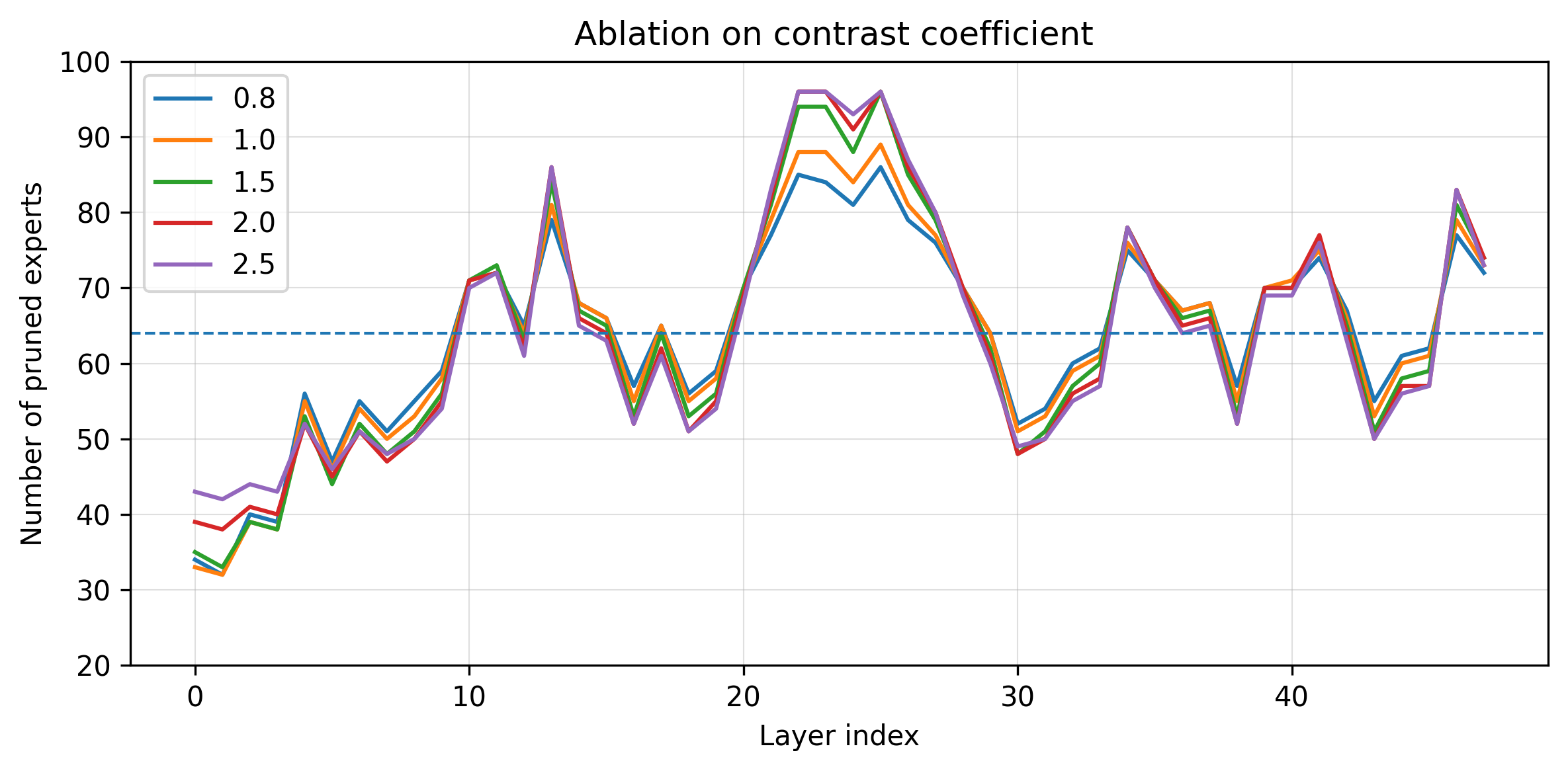}\hfill
    \includegraphics[width=.48\textwidth]{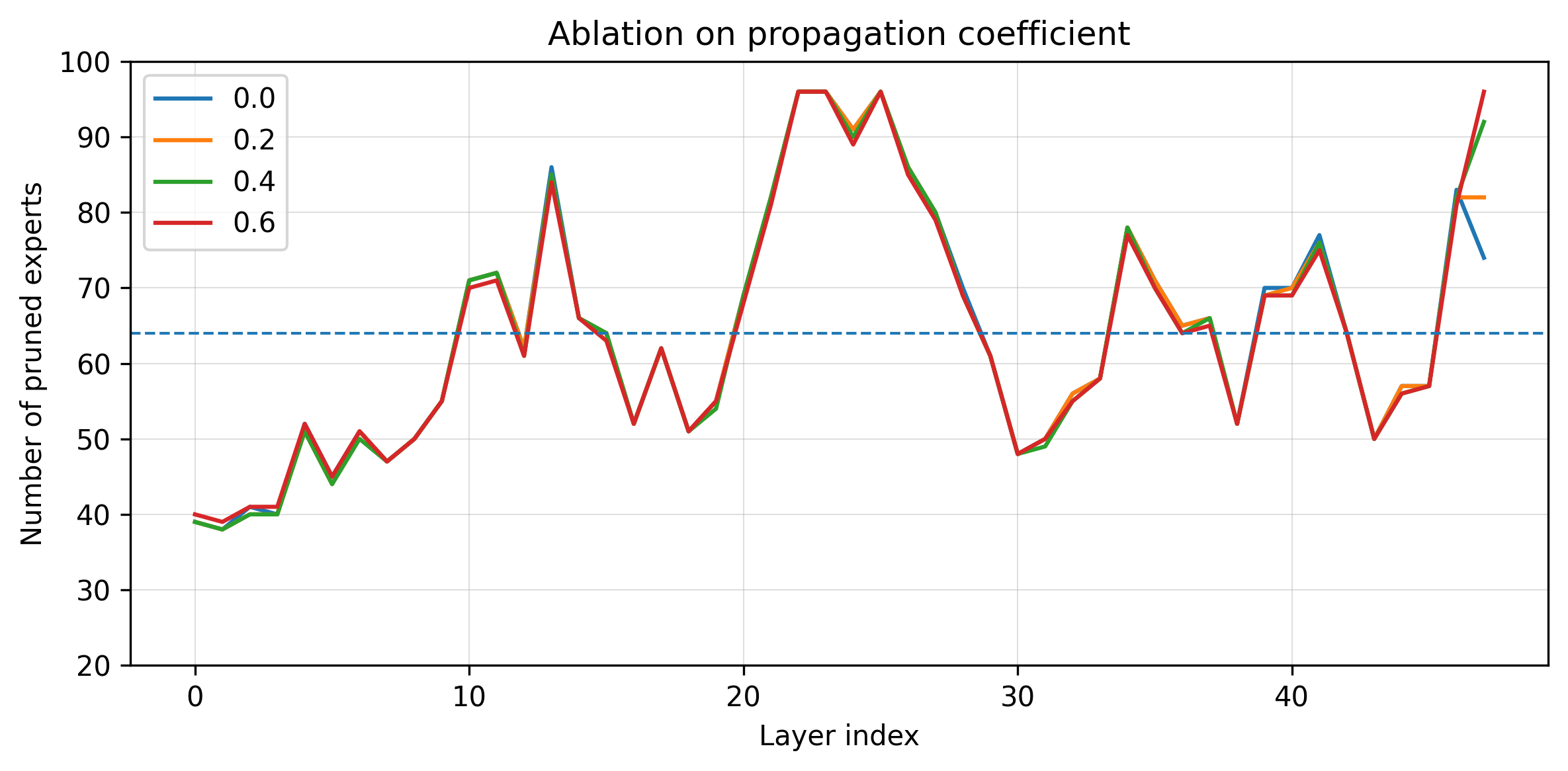}
    \caption{Layer-wise allocation profiles on Qwen3-30B-A3B with EASY-EP at 50\% global expert pruning. Left: varying $p$ with $\beta=0$ changes the contrast between more- and less-pruned layers. Right: varying $\beta\in[0,0.6]$ with $p=2.0$ produces similar profiles.}
    \label{fig:ablation_profiles}
\end{figure*}

\subsection{Scalability to Large-scale Reasoning MoEs}
\label{sec:exp:large_scale}

At 50\% global expert pruning, we compare \ours and uniform allocation on \textbf{Qwen3-235B-A22B} and \textbf{DeepSeek-R1}, using EASY-EP for both (Table~\ref{tab:large_scale_layerwise_compact}). On both models, \ours narrows the gap to the unpruned model in the unweighted three-score mean. It raises this mean from 80.1 to 82.8 on DeepSeek-R1, compared with 83.4 before pruning, and from 86.4 to 87.9 on Qwen3-235B-A22B. On DeepSeek-R1, GPQA-Diamond improves by 5.5 points and LiveCodeBench by 3.0 points over uniform allocation, while MATH-500 declines by 0.4 points. Deployment efficiency is reported in Appendix~\ref{app:impl}.
% \subsection{Scalability to Large-scale Reasoning MoEs}
% \label{sec:exp:large_scale}

% We evaluate scalability on \textbf{Qwen3-235B} and \textbf{DeepSeek-R1} under a strict 50\% pruning regime (Table~\ref{tab:large_scale_layerwise_compact}). While uniform allocation degrades DeepSeek-R1's average reasoning score by 3.3\%, SPARE achieves near-lossless compression ($-0.6\%$ gap), substantially outperforming the uniform baseline by \textbf{+5.5\%} on GPQA. 
% In addition to these algorithmic improvements, the 50\% reduction in the number of experts results nearly 2$\times$ reduction in the model memory footprint. This compression enables the deployments from multi-node configurations (e.g., 2 nodes) to a single A100 node, thereby eliminating the high-latency inter-node all-to-all communication overhead inherent to Expert Parallelism (EP). Detailed communication cost analysis is shown in Appendix~\ref{sec:appendix_ep_comm_savings}.

\subsection{Comparison with Layer-wise Allocation Baselines}
\label{sec:exp:layerwise_comp}
We compare \ours with OWL~\citep{yin2024outlier} adapted from weight-sparsity allocation to expert pruning, holding EASY-EP fixed for both methods. On Qwen3-30B-A3B, average accuracy is comparable at 12.5\% pruning and higher with \ours at 25\%, 37.5\%, and 50\% (Appendix~\ref{app:owl_pareto}). Direct comparisons with other MoE-specific allocation pipelines remain outside the evaluation.

\subsection{Calibration Robustness}
\label{sec:calibration}
We compare DeepSeek-R1 layer-average entropy across four calibration domains: C4 (web), MATH-500 (math), LiveCodeBench (code), and GPQA-Diamond (science). Table~\ref{tab:calibration} reports correlations relative to C4. Spearman correlations range from 0.958 to 0.974, indicating similar layer rankings across these domains. This analysis measures entropy-profile stability; it does not establish identical integer budgets or unchanged downstream performance under every calibration choice.

\begin{table}[t]
\centering
\small
\setlength{\tabcolsep}{3.5pt}
\caption{Cross-domain correlations of layer-average routing entropy on DeepSeek-R1, using C4 as the reference calibration domain.}
\label{tab:calibration}
\resizebox{\columnwidth}{!}{%
\begin{tabular}{llcc}
\toprule
Model & Calibration analysis & Pearson $r$ & Spearman $\rho$  \\
\midrule
DeepSeek-R1 & C4 vs. MATH-500 & 0.988 & 0.962  \\
DeepSeek-R1 & C4 vs. LiveCodeBench & 0.991 & 0.974  \\
DeepSeek-R1 & C4 vs. GPQA-Diamond & 0.979 & 0.958  \\

\bottomrule
\end{tabular}
}
\end{table}

\subsection{Ablation on Hyper-parameters}
\label{sec:ablation}
We run two separate sweeps on Qwen3-30B-A3B with EASY-EP at 50\% global expert pruning (Table~\ref{tab:ablate_avg_two_rows}). The $p$ sweep fixes $\beta=0$: mean accuracy peaks at $p=2.0$ (63.2\%) and falls to 58.5\% at $p=2.5$, showing that stronger contrast need not help. The $\beta$ sweep fixes $p=2.0$: scores range from 63.1\% to 63.8\%, peaking at $\beta=0.6$. These are separate sweeps, not a joint parameter grid.

The allocation profiles in Figure~\ref{fig:ablation_profiles} show that larger $p$ increases the contrast between more- and less-pruned layers, while the profiles for $\beta\in[0,0.6]$ remain close. We selected $p=2.0$ and $\beta=0.6$ on the reported Qwen3-30B-A3B evaluation benchmarks and fixed them for all other models.

\begin{table}[t]
\centering
\caption{\textbf{Hyper-parameter ablation.} Qwen3-30B-A3B with EASY-EP at 50\% global pruning. Unweighted eight-task mean accuracy (\%); the $p$ sweep fixes $\beta=0$, and the $\beta$ sweep fixes $p=2.0$. Bold marks the best score within each sweep.}
\label{tab:ablate_avg_two_rows}
\scriptsize
\setlength{\tabcolsep}{4pt}
\renewcommand{\arraystretch}{0.95}
\resizebox{\columnwidth}{!}{%
\begin{tabular}{c|ccccc|c|cccc}
\toprule
 & \multicolumn{5}{c|}{$p$} &  & \multicolumn{4}{c}{$\beta$ } \\
\cmidrule(lr){2-6}\cmidrule(lr){8-11}
\textbf{Param} & 0.8 & 1.0 & 1.5 & 2.0 & 2.5 & \textbf{Param} & 0.0 & 0.2 & 0.4 & 0.6 \\
\midrule
\textbf{Avg.$\uparrow$} 
& 61.9 & 61.5 & 62.8 & \textbf{63.2} & 58.5
& \textbf{Avg.$\uparrow$}
& 63.2 & 63.1 & 63.1 & \textbf{63.8} \\
\bottomrule
\end{tabular}}
\end{table}

% \subsection{Inference Efficiency}
% \label{sec:exp:efficiency}

% \textcolor{red}{
% We next examine whether the structural reduction brought by \ours translates into practical
% inference benefits. Since expert pruning physically removes routed experts, it directly reduces
% the MoE parameter footprint and can lower the required expert-parallel deployment scale.
% Table~\ref{tab:efficiency} summarizes the deployment cost under 50\% expert pruning.
% Compared with the unpruned model, \ours reduces the number of routed experts by half,
% leading to an approximately 2$\times$ reduction in routed-expert memory.
% For large-scale MoEs such as DeepSeek-R1, this compression can reduce the deployment
% from multi-node expert-parallel serving to a single-node configuration, thereby avoiding
% inter-node all-to-all communication. Importantly, these efficiency gains are achieved while
% preserving near-lossless reasoning performance, as shown in Table~\ref{tab:large_scale_layerwise_compact}.
% }

\section{Conclusion and Future Work}
\label{sec:conclusion}
We introduced \ours, a training-free method that uses token-level routing entropy to allocate expert-pruning budgets across MoE language-model layers and can be paired with existing intra-layer selectors. A perturbation analysis and controlled interventions with a fixed selector associate concentrated routing with greater pruning risk. Under the same expert selector and global budget, \ours improves average performance in the evaluated settings, with larger gains at high pruning ratios: at 50\% expert pruning, it raises the Qwen3-30B-A3B eight-task average from 57.2\% to 63.8\%, and the DeepSeek-R1 three-benchmark average from 80.1 to 82.8, close to the unpruned 83.4. Future work includes studying calibration shift and combining entropy-guided allocation with quantization and expert merging.

\clearpage
\section*{Limitations}
SPARE yields its largest gains at aggressive pruning ratios; at moderate ratios such as 25\%, its improvements over uniform allocation are small. We study SPARE only for expert pruning; combining it with quantization or expert merging is left for future work.

\bibliography{references}

%%%%%%%%%%%%%%%%%%%%%%%%%%%%%%%%%%%%%%%%%%%%%%%%%%%%%%%%%%%%

\newpage
\appendix

% \section{Technical appendices and supplementary material}
% Technical appendices with additional results, figures, graphs, and proofs may be submitted with the paper submission before the full submission deadline (see above). You can upload a ZIP file for videos or code, but do not upload a separate PDF file for the appendix. There is no page limit for the technical appendices. 

% Note: Think of the appendix as ``optional reading'' for reviewers. The paper must be able to stand alone without the appendix; for example, adding critical experiments that support the main claims to an appendix is inappropriate. 

\section{Additional results}
\label{app:impl}

% \subsection{Additional Experiments on GLM-5}
% \label{app:glm5}

To further validate the applicability of \ours to fine-grained MoE architectures, we extend our evaluation to the recent GLM-5. 
Our main experiments already cover representative fine-grained MoEs, including Qwen3-30B-A3B, Qwen3-235B-A22B, and DeepSeek-R1. 
As summarized in Table~\ref{tab:app_moe_arch_configs}, these models differ in the number of routed experts, the use of shared experts, and the number of MoE layers. 
GLM-5 follows a structure similar to DeepSeek-R1, with 256 routed experts, top-8 routing, and one shared expert, but increases the number of layers from 61 to 78.

\begin{table}[ht]
\centering
\caption{\textbf{Architectural configurations of fine-grained MoEs.}
We compare the routed expert count, routing top-$K$, shared expert design, and layer count across the evaluated MoE models.}
\label{tab:app_moe_arch_configs}
\begingroup
\scriptsize
\setlength{\tabcolsep}{4pt}
\renewcommand{\arraystretch}{0.9}
\resizebox{0.78\columnwidth}{!}{%
\begin{tabular}{l|c|c|c|c}
\toprule
\textbf{Model} & \textbf{Routed Experts} & \textbf{Top-$K$} & \textbf{Shared} & \textbf{Layers} \\
\midrule
DeepSeek-R1 (671B) & 256 & 8 & Yes (1) & 61 \\
Qwen3-30B-A3B & 128 & 8 & No & 48 \\
Qwen3-235B-A22B & 128 & 8 & No & 94 \\
GLM-5 & 256 & 8 & Yes (1) & 78 \\
\bottomrule
\end{tabular}}
\endgroup
\end{table}

As in Section~\ref{sec:exp:large_scale}, both \ours and uniform allocation use EASY-EP at 50\% global expert pruning.
On GLM-5, \ours improves the average accuracy from 87.5\% to 89.8\%, yielding a \textbf{+2.3} point gain over uniform allocation. 
This additional result extends the observed average improvement to another fine-grained MoE architecture.

\paragraph{Deployment efficiency.} 
At the same global pruning budget, SPARE and uniform allocation remove the same total number of experts, giving equal routed-expert parameter storage when expert sizes and precision match. We use modified fused MoE Triton kernels to support variable per-layer expert counts. Table~\ref{tab:efficienty} reports throughput on a DeepSeek-style architecture at 50\% pruning. Relative differences from uniform allocation range from a 1.16\% decrease to a 0.56\% increase across the tested batch sizes. Throughput is similar in these measured configurations.

\begin{table}[t]
\centering
\caption{\textbf{Throughput under different batch sizes.}}
\label{tab:efficienty}
\begingroup
\tiny
\setlength{\tabcolsep}{3pt}
\renewcommand{\arraystretch}{0.85}
\resizebox{\columnwidth}{!}{%
\begin{tabular}{l|cccc}
\toprule
\textbf{Batch size} & \textbf{1} & \textbf{16} & \textbf{32} & \textbf{64} \\
\midrule
Uniform & 8.61 & 172.32 & 181.54 & 162.31 \\
SPARE   & 8.51 & 172.57 & 182.56 & 162.24 \\
\bottomrule
\end{tabular}}
\endgroup
\end{table}

Additionally, non-uniform allocation opens future optimization possibilities: layers retaining more experts can use padded expert placement to co-locate the most frequently activated experts, potentially reducing HBM transfer beyond what uniform pruning achieves.

\section{Comparison with Layer-wise Allocation Methods}
\label{app:owl_pareto}

We compare \ours with the adapted OWL allocation on Qwen3-30B-A3B, holding EASY-EP fixed for both methods. Figure~\ref{fig:pareto} plots average benchmark accuracy against the global expert-pruning ratio.
The methods are comparable at 12.5\% pruning; \ours has higher average accuracy at 25\%, 37.5\%, and 50\%. OWL uses activation-aware weight outlier statistics, whereas SPARE uses token-level routing entropy.

\begin{figure}[t]
    \centering
    \includegraphics[width=0.7\linewidth]{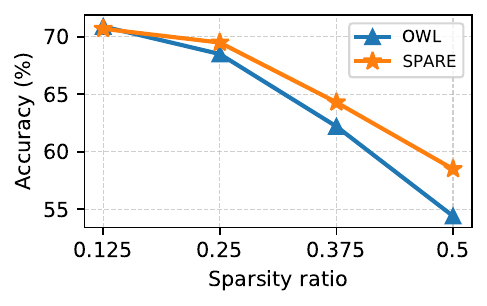}
    \vspace{-2mm}
    \caption{Average benchmark accuracy versus global expert-pruning ratio on Qwen3-30B-A3B. }
    \label{fig:pareto}
\end{figure}

\section{Proof of Main Theorem}
\label{app:proof}
\subsection{Notation and Mathematical Definitions}
\subsubsection{Data and Tokens}
Let $\mathcal{D}$ denote the target token distribution (e.g., pruning-evaluation domain).
We write $x\sim\mathcal{D}$ for a random token representation entering a given MoE layer.
In practice, $\mathcal{D}$ is approximated by a finite calibration set of $M$ tokens,
$D=\{x_1,\ldots,x_M\}$.

\subsubsection{MoE Layer, Experts, and Router}
Consider a fixed MoE layer $\ell$ with:
\begin{itemize}
  \item $N$ experts indexed by $j\in\{1,\ldots,N\}$, each mapping token input $x$ to an output vector
        $f_{\ell,j}(x)\in\mathbb{R}^d$ (dimension $d$ depends on the model).
  \item Router logits $z_\ell(x)=(z_{\ell,1}(x),\ldots,z_{\ell,N}(x))\in\mathbb{R}^N$.
\end{itemize}

\subsubsection{Top-K Routing and Gating Distribution}
Let $K<N$ be the number of selected experts per token.
Define the Top-K index set
\begin{equation}
T_\ell(x)\;\triangleq\;\operatorname{TopK}\bigl(z_\ell(x)\bigr),
  \qquad |T_\ell(x)|=K,
\end{equation}
where $\operatorname{TopK}(\cdot)$ returns indices of the $K$ largest components, using a fixed ordering of expert indices to break ties before and after pruning.

The Top-K gating distribution $g_\ell(x)$ is defined by softmax restricted to $T_\ell(x)$:
\begin{equation}
  g_{\ell,j}(x)
  \;\triangleq\;
  \begin{cases}
    \displaystyle
    \frac{e^{z_{\ell,j}(x)}}{\sum_{t\in T_\ell(x)} e^{z_{\ell,t}(x)}},
    & j\in T_\ell(x),\\[10pt]
    0, & j\notin T_\ell(x).
  \end{cases}
  \label{eq:app-topk-gate}
\end{equation}
Thus $g_{\ell,j}(x)\ge 0$ and $\sum_{j=1}^N g_{\ell,j}(x)=1$, but only $K$ entries are nonzero.

\subsubsection{Layer Output (MoE Branch)}
The MoE layer output (restricted to the expert mixture branch) is
\begin{equation}
\begin{aligned}
  y_\ell(x)
  &\triangleq\sum_{j=1}^N g_{\ell,j}(x)\,f_{\ell,j}(x)\\
  &=\sum_{t\in T_\ell(x)} g_{\ell,t}(x)\,f_{\ell,t}(x).
\end{aligned}
  \label{eq:layer-output}
\end{equation}

\subsubsection{Pruning Operation and Top-K Rerouting}
\label{app:pruning_renorm}

Let $P\subseteq\{1,\dots,N\}$ denote the removed expert set, with $|P|=m\le N-K$.
Router logits are unchanged by pruning. We recompute Top-$K$ over the remaining experts, retaining their original indices:
\[
\begin{aligned}
T_{\ell,P}(x)&\triangleq
\operatorname{TopK}_{j\notin P}\!\bigl(z_{\ell,j}(x)\bigr),\\
|T_{\ell,P}(x)|&=K.
\end{aligned}
\]
This uses the same tie-breaking rule as the original routing. The removed experts that were originally active and their gate mass are
\begin{equation}
\begin{aligned}
D_\ell(x) &\triangleq T_\ell(x)\cap P,\\
\mathcal{E}_\ell(x,P)&\triangleq
\sum_{j\in D_\ell(x)} g_{\ell,j}(x),
\end{aligned}
\label{eq:pruned_mass_def}
\end{equation}
where $\mathcal{E}_\ell(x,P)$ is the total probability mass assigned to pruned experts within the original Top-$K$ support.

The post-pruning gating distribution is softmax over the recomputed support:
\begin{equation}
g_{\ell,j}^{(P)}(x) \triangleq
\begin{cases}
\dfrac{e^{z_{\ell,j}(x)}}{\sum_{t\in T_{\ell,P}(x)}e^{z_{\ell,t}(x)}}, & j\in T_{\ell,P}(x),\\[10pt]
0, & \text{otherwise},
\end{cases}
\label{eq:renorm_gating_def}
\end{equation}
Since $N-m\ge K$, this remains well defined even when every originally active expert is removed.
The post-pruning output is $y_\ell^{(P)}(x)=\sum_j g_{\ell,j}^{(P)}(x)f_{\ell,j}(x)$, and the perturbation is $\delta_\ell(x,P)=y_\ell^{(P)}(x)-y_\ell(x)$.

\subsubsection{Routing Entropy and Concentration}
For a Top-K sparse distribution $g_\ell(x)$, define Shannon entropy
\begin{equation}
  H\!\left(g_\ell(x)\right)\;\triangleq\;-\sum_{j=1}^N g_{\ell,j}(x)\log g_{\ell,j}(x),
  \label{eq:def-entropy}
\end{equation}
where the convention $0\log 0=0$ is used.
Since only $K$ entries are nonzero, we have $0\le H(g_\ell(x))\le \log K$.

Define the layer-average entropy and routing entropy gap as in Eq.~\eqref{eq:specialization_def}.
Equivalently, let $u_{T_\ell(x)}$ denote the uniform distribution on the Top-K support $T_\ell(x)$:
$u_{T_\ell(x)}(j)=1/K$ if $j\in T_\ell(x)$ and $0$ otherwise. Then
\begin{equation}
\begin{aligned}
  \mathrm{KL}\!\left(g_\ell(x)\,\middle\|\,u_{T_\ell(x)}\right)
  &=\log K-H\!\left(g_\ell(x)\right),\\
  S_\ell
  &=\mathbb{E}_{x}\!\left[
  \mathrm{KL}\!\left(g_\ell(x)\,\middle\|\,u_{T_\ell(x)}\right)
  \right].
\end{aligned}
  \label{eq:entropy-gap-equals-kl}
\end{equation}
\textbf{Interpretation.} Larger $S_\ell$ means greater routing concentration: the gating distribution is further from uniform on its active support.

\subsection{Assumptions}
We state the assumption used to relate routing concentration to pruning perturbation.
% All constants are allowed to be layer-dependent.

% \paragraph{Assumption A0 (Top-K routing and rerouting).}
% Routing is Top-K selection followed by softmax on the selected set, as in \eqref{eq:topk-gate}.
% After pruning set $P$, rerouting recomputes Top-K among remaining experts and re-applies softmax
% as in \eqref{eq:postprune-gate}. Router logits $z_\ell(x)$ are unchanged by pruning.

\paragraph{Assumption 1 (Output scale control).}
For each layer $\ell$, there exists $B_\ell>0$ such that
\begin{equation}
  \|f_{\ell,j}(x)\|_1 \le B_\ell
  \quad \forall x,P,\ j\in T_\ell(x)\cup T_{\ell,P}(x),
  \label{eq:assump-bounded}
\end{equation}
where $P$ ranges over admissible pruning sets with $|P|\le N-K$.

\subsection{Single-Layer Pruning Perturbation}
\label{appendix:single_layer}

\paragraph{Proofs for general and deterministic selectors.}
\label{app:general_selector_proofs}
For Proposition~\ref{prop:general_selector}, the triangle inequality, bounded expert outputs, and the routing identity proved below give, pointwise for every admissible $x$ and $P$,
\[
\begin{aligned}
\|\delta_\ell(x,P)\|_1
&\le B_\ell\|g_\ell^{(P)}(x)-g_\ell(x)\|_1\\
&=2B_\ell\mathcal{E}_\ell(x,P).
\end{aligned}
\]
Squaring before taking expectations yields Eqs.~\eqref{eq:maskwise_perturbation} and~\eqref{eq:general_selector_bound} for any mask distribution $q_\ell$. A deterministic selector is the special case with all probability on its returned mask. All expectations are conditional on fixed calibration data and are taken over the stated evaluation-token distribution.

For Corollary~\ref{cor:selector_transfer}, the bracketed term in Eq.~\eqref{eq:entropy_bound_exact} is
\begin{equation}
\begin{aligned}
\Psi_\ell(m)={}&c_0(m,N)\\
&+c_1(m,N)\left[\frac{1}{K}+2S_\ell\right].
\end{aligned}
\label{eq:entropy_mass_envelope}
\end{equation}
Lemmas~\ref{lemma:3} and~\ref{lemma:4} give
\[
\begin{aligned}
\mathcal{M}_\ell(U_{\ell,m})
&=c_0(m,N)+c_1(m,N)\mathbb{E}_x[Q_\ell(x)]\\
&\le\Psi_\ell(m).
\end{aligned}
\]
Combining this with Eqs.~\eqref{eq:maskwise_perturbation} and~\eqref{eq:selector_dominance} proves Eq.~\eqref{eq:deterministic_selector_bound}. In particular, a selector that minimizes $\mathbb{E}_x[\mathcal{E}_\ell(x,P)^2]$ over the finite mask family satisfies the condition with $\varepsilon_{\mathcal A,\ell}=0$, since its minimum cannot exceed the uniform average. This observation does not establish the condition for selectors optimizing a different saliency criterion. A residual measured on calibration samples is an empirical estimate and does not by itself certify the condition on the evaluation distribution.

\paragraph{Proof Sketch for Main Theorem.}
Pruning a subset $P$ changes the Top-$K$ routing distribution from $g_\ell(x)$ to $g^{(P)}_\ell(x)$.
Let $\mathcal{E}_\ell(x,P)$ denote the total gate mass assigned to experts that are both pruned and originally selected.
A key identity shows that the routing change is exact:
$\|g^{(P)}_\ell(x)-g_\ell(x)\|_1 = 2\mathcal{E}_\ell(x,P)$ (Lemma~\ref{lemma:1}).
Since the MoE output is a convex combination of bounded expert outputs, the output perturbation satisfies
$\|\delta_\ell(x,P)\|_1^2 \le 4B_\ell^2 \mathcal{E}_\ell(x,P)^2$ (Lemma~\ref{lemma:2_1}).
Under uniform random pruning, $\mathbb{E}_P[\mathcal{E}_\ell(x,P)^2]$ decomposes into a constant term and a collision term
proportional to $Q_\ell(x)=\sum_j g_{\ell,j}(x)^2$ (Lemma~\ref{lemma:3}).
Finally, for Top-$K$ sparse routing, the collision probability is upper bounded by the entropy gap
$Q_\ell(x)\le \frac{1}{K}+2(\log K-H(g_\ell(x)))$ (Lemma~\ref{lemma:4}),
which yields the stated bound and establishes monotonic dependence on routing entropy gap.

\begin{lemma}[Exact $\ell_1$ identity]
\label{lemma:1}
Fix layer $\ell$, token $x$, and pruned set $P$.
Let $D_\ell(x)\triangleq T_\ell(x)\cap P$ be the set of pruned experts that were originally in Top-K,
and define $\mathcal{E}_\ell(x,P)=\sum_{j\in D_\ell(x)} g_{\ell,j}(x)$.
Then
\begin{equation}
  \bigl\|g^{(P)}_\ell(x)-g_\ell(x)\bigr\|_1\;=\;2\,\mathcal{E}_\ell(x,P).
  \label{eq:lemma1}
\end{equation}
\end{lemma}

\begin{proof}
For brevity, suppress layer $\ell$ and token $x$ in notation.
Let $T\triangleq T_\ell(x)$ and $T_P\triangleq T_{\ell,P}(x)$.
Define
\[
D\triangleq T\cap P,\quad
U\triangleq T\setminus D,\quad
R\triangleq T_P\setminus U.
\]
Unchanged logits and consistent tie-breaking imply $U\subseteq T_P$.
Thus $T=U\cup D$ and $T_P=U\cup R$, with $|R|=|D|$ since both supports have size $K$.

Define the Top-K softmax normalizers:
\[
Z \triangleq \sum_{t\in T} e^{z_t},\qquad
Z_P \triangleq \sum_{t\in T_P} e^{z_t}.
\]
For any $u\in U$, we have
\[
\begin{aligned}
g_u &= \frac{e^{z_u}}{Z},\\
g^{(P)}_u &= \frac{e^{z_u}}{Z_P}=\alpha\,g_u,
\qquad \alpha\triangleq \frac{Z}{Z_P}.
\end{aligned}
\]
The factor $\alpha$ is common to all $u\in U$ because softmax shares the denominator $Z_P$; if $U$ is empty, the statements about its entries are vacuous.

Next, since $T$ consists of the top-$K$ logits, every $r\in R$ satisfies
$z_r \le \min_{t\in T} z_t$, while every $d\in D$ satisfies $z_d \ge \min_{t\in T} z_t$.
Because $|R|=|D|$, it follows that
\[
\begin{aligned}
\sum_{r\in R} e^{z_r}&\le \sum_{d\in D} e^{z_d}
\quad\Rightarrow\\
Z_P &= \sum_{u\in U} e^{z_u} + \sum_{r\in R} e^{z_r}\\
&\le \sum_{u\in U} e^{z_u} + \sum_{d\in D} e^{z_d}=Z.
\end{aligned}
\]
Hence $\alpha=Z/Z_P\ge 1$.

Now compute the $\ell_1$ difference by partitioning indices into $D$, $U$, $R$ and the remainder:
\begin{itemize}
  \item For $d\in D$, $g_d>0$ but $g^{(P)}_d=0$, so $|g^{(P)}_d-g_d|=g_d$.
  \item For $u\in U$, $g^{(P)}_u=\alpha g_u\ge g_u$, so $|g^{(P)}_u-g_u|=(\alpha-1)g_u$.
  \item For $r\in R$, $g_r=0$ and $g^{(P)}_r>0$, so $|g^{(P)}_r-g_r|=g^{(P)}_r$.
  \item All other indices have both weights $0$ and contribute nothing.
\end{itemize}
Therefore,
\[
\|g^{(P)}-g\|_1
= \sum_{d\in D} g_d + \sum_{u\in U}(\alpha-1)g_u + \sum_{r\in R} g^{(P)}_r.
\]
By definition, we have $L\triangleq \sum_{d\in D} g_d$.
Because $\sum_j g_j=1$ and $\sum_j g^{(P)}_j=1$, and $g^{(P)}_d=0$ for $d\in D$, we have
\[
\begin{aligned}
\sum_{i\notin D} \bigl(g^{(P)}_i-g_i\bigr)
&= \left(\sum_{i} g^{(P)}_i\right)
- \left(\sum_{d\in D} g^{(P)}_d\right)\\
&\quad- \left[\left(\sum_{i} g_i\right)
- \left(\sum_{d\in D} g_d\right)\right]\\
&=(1-0)-(1-L)=L.
\end{aligned}
\]
Moreover, for $i\notin D$ we have $g^{(P)}_i\ge g_i$ (since $U$ entries scale up and $R$ entries become positive),
so
\[
\sum_{i\notin D} |g^{(P)}_i-g_i|
= \sum_{i\notin D} (g^{(P)}_i-g_i)
= L.
\]
Combining with $\sum_{d\in D}|g^{(P)}_d-g_d|=\sum_{d\in D}g_d=L$ yields
\[
\|g^{(P)}-g\|_1 = L + L = 2L,
\]
which proves the claim.
\end{proof}

\paragraph{Intuition.}
The pruned experts in $D$ lose exactly $L$ probability mass (their original total gate mass),
and because the gating distribution must still sum to $1$, the remaining indices must gain
exactly $L$ mass in total. Since the gain occurs only on indices where $g^{(P)}\ge g$
(common scaling on retained Top-K plus newly activated replacements), the $\ell_1$ distance
equals (loss $L$) + (gain $L$) $=2L$.

\begin{lemma}[Perturbation bound: bounded-norm variant]
\label{lemma:2_1}
Under Assumption 1 in \eqref{eq:assump-bounded}, for any pruned set $P$, we have
\begin{equation}
\begin{aligned}
  \mathbb{E}_{x}[\|\delta_\ell(x,P)\|_1]
  &\le B_\ell\,\mathbb{E}_{x}
  \bigl[\|g^{(P)}_\ell(x)-g_\ell(x)\|_1\bigr]\\
  &=2B_\ell\,\mathbb{E}_{x}
  [\mathcal{E}_\ell(x,P)].
\end{aligned}
  \label{eq:lemma2-bounded}
\end{equation}
Consequently,
\begin{equation}
\begin{aligned}
  \mathbb{E}_{x}[\|\delta_\ell(x,P)\|_1^2]
  &\le 4B_\ell^2\,\mathbb{E}_{x}
  [\mathcal{E}_\ell(x,P)^2].
\end{aligned}
  \label{eq:lemma2-bounded-sq}
\end{equation}
\end{lemma}

\begin{proof}
By definition,
\[
\delta_\ell(x,P)=\sum_{j=1}^N \bigl(g^{(P)}_{\ell,j}(x)-g_{\ell,j}(x)\bigr) f_{\ell,j}(x).
\]
Apply triangle inequality we have:
\[
\|\delta_\ell(x,P)\|_1
\le \sum_{j=1}^N |g^{(P)}_{\ell,j}(x)-g_{\ell,j}(x)|\,\|f_{\ell,j}(x)\|_1
\]
Assumption 1 gives the pointwise inequality
$\|\delta_\ell(x,P)\|_1\le B_\ell\|g^{(P)}_\ell(x)-g_\ell(x)\|_1=2B_\ell\mathcal{E}_\ell(x,P)$.
Taking expectation over $x$ yields
\[
\mathbb{E}_{x}[\|\delta_\ell(x,P)\|_1]
\le B_\ell\,\mathbb{E}_{x}
[\|g^{(P)}_\ell(x)-g_\ell(x)\|_1].
\]
The equality in \eqref{eq:lemma2-bounded} follows from Lemma~\ref{lemma:1}.
Squaring the pointwise inequality before taking expectation yields \eqref{eq:lemma2-bounded-sq}.
\end{proof}

\iffalse
\begin{lemma}[Perturbation bound: second-moment variant]
\label{lemma:2_2}
Under Assumption \eqref{eq:assump-second-moment}, for any pruned set $P$,
\begin{equation}
  \mathbb{E}_{x}\!\left[\|\delta_\ell(x,P)\|^2\right]
  \;\le\;
  4\,\mathbb{E}_{x}\!\left[\mathcal{E}_\ell(x,P)^2\cdot \sum_{j\in T_\ell(x)\cup T_{\ell,P}(x)}\|f_{\ell,j}(x)\|^2\right]
  \;\le\;
  4\,\Sigma_\ell^2\cdot \mathbb{E}_x\!\left[\mathcal{E}_\ell(x,P)^2\right].
  \label{eq:lemma2-secondmoment}
\end{equation}
\end{lemma}

\begin{proof}
Let $\Delta g_{\ell,j}(x)\triangleq g^{(P)}_{\ell,j}(x)-g_{\ell,j}(x)$.
Then $\Delta g_{\ell,j}(x)\neq 0$ only on indices in $T_\ell(x)\cup T_{\ell,P}(x)$.
Using Cauchy--Schwarz,
\[
\|\delta_\ell(x,P)\|
= \left\|\sum_{j} \Delta g_{\ell,j}(x) f_{\ell,j}(x)\right\|
\le \sum_j |\Delta g_{\ell,j}(x)|\,\|f_{\ell,j}(x)\|
\le \|\Delta g_\ell(x)\|_2\cdot \left(\sum_{j\in T_\ell(x)\cup T_{\ell,P}(x)}\|f_{\ell,j}(x)\|^2\right)^{1/2}.
\]
Square both sides and use $\|\Delta g\|_2 \le \|\Delta g\|_1 = 2\mathcal{E}_\ell(x,P)$ (Lemma \ref{lemma:1}):
\[
\|\delta_\ell(x,P)\|^2
\le \|\Delta g_\ell(x)\|_2^2 \sum_{j\in T_\ell(x)\cup T_{\ell,P}(x)}\|f_{\ell,j}(x)\|^2
\le (2\mathcal{E}_\ell(x,P))^2 \sum_{j\in T_\ell(x)\cup T_{\ell,P}(x)}\|f_{\ell,j}(x)\|^2.
\]
Taking expectation over $x$ and applying Assumption \eqref{eq:assump-second-moment} yields the claim.
\end{proof}

\fi

\paragraph{Intuition.}
The MoE output is a convex combination of expert outputs. Pruning changes only the mixing weights.
Thus, the output perturbation is controlled by (i) how much the gating distribution changes
($\|g^{(P)}-g\|$), and (ii) the scale of expert outputs. Lemma \ref{lemma:1} makes (i) exact in terms of the
pruned gate mass $L$.

% \subsection{Lemma 3: Random Pruning Moment and Collision Mass}
% Define the collision mass
% \begin{equation}
%   Q_\ell(x)\;\triangleq\;\sum_{j=1}^N g_{\ell,j}(x)^2.
%   \label{eq:def-collision}
% \end{equation}
% We have the following lemma:
\begin{lemma}[uniform random pruning]
\label{lemma:3}
% Exact $\mathbb{E}\left[L^2\right]$ for uniform random pruning.
Fix $m\in\{1,\ldots,N-1\}$. Let $P$ be drawn uniformly at random from all subsets of size $m$.
Then for any fixed $x$,
\begin{equation}
\begin{gathered}
  \mathbb{E}_{P}\!\left[\mathcal{E}_\ell(x,P)^2 \,\middle|\, x\right]\\
  =c_0(m,N)+c_1(m,N)\,Q_\ell(x),
\end{gathered}
  \label{eq:EL2}
\end{equation}
where
\begin{equation}
\begin{aligned}
  c_0(m,N)&\triangleq\frac{m(m-1)}{N(N-1)},\\
  c_1(m,N)&\triangleq\frac{m(N-m)}{N(N-1)},\\
  Q_\ell(x)&\triangleq\sum_{j=1}^N g_{\ell,j}(x)^2.
\end{aligned}
  \label{eq:c0c1}
\end{equation}
\end{lemma}

\begin{proof}
Let $I_j=\mathbf{1}\{j\in P\}$ be the inclusion indicator. Then
\[
\mathcal{E}_\ell(x,P)=\sum_{j=1}^N I_j\,g_{\ell,j}(x).
\]
Therefore
\[
\mathcal{E}_\ell(x,P)^2 = \sum_{j} I_j g_j^2 + 2\sum_{j<k} I_j I_k g_j g_k,
\]
where we suppress $\ell,x$ in $g$ for readability.
For uniform sampling without replacement, $\mathbb{E}[I_j]=m/N$ and
$\mathbb{E}[I_j I_k]=\frac{m(m-1)}{N(N-1)}$ for $j\neq k$.
Taking expectation over $P$ conditional on $x$:
\[
\mathbb{E}_P[L^2|x]
= \frac{m}{N}\sum_j g_j^2
+ 2\cdot\frac{m(m-1)}{N(N-1)}\sum_{j<k} g_j g_k.
\]
Note that $\sum_{j<k} g_j g_k = \frac{1}{2}\left[(\sum_j g_j)^2 - \sum_j g_j^2\right]
= \frac{1}{2}(1 - Q)$ since $\sum_j g_j=1$.
Hence
\[
\begin{aligned}
\mathbb{E}_P[L^2|x]
&= \frac{m}{N}Q + \frac{m(m-1)}{N(N-1)}(1-Q)\\
&= \frac{m(m-1)}{N(N-1)}\\
&\quad+ \left[\frac{m}{N}-\frac{m(m-1)}{N(N-1)}\right]Q.
\end{aligned}
\]
Simplifying the coefficient yields $c_1(m,N)=\frac{m(N-m)}{N(N-1)}$, proving \eqref{eq:EL2}.
\end{proof}

\paragraph{Intuition.}
Randomly pruning $m$ experts is like sampling $m$ indices. The second moment of the removed
probability mass depends on how concentrated $g$ is. Concentration is captured by collision mass $Q=\|g\|_2^2$.

% \subsection{Lemma 4: Collision Mass Upper Bounded by Entropy Gap}
\begin{lemma}[Entropy--collision bound for Top-K sparse distributions]
\label{lemma:4}
For any Top-K sparse distribution $g_\ell(x)$ (support size at most $K$),
\begin{equation}
\begin{aligned}
  Q_\ell(x)&=\sum_j g_{\ell,j}(x)^2\\
  &\le\frac{1}{K}+2\Bigl(\log K-H(g_\ell(x))\Bigr).
\end{aligned}
  \label{eq:Q-vs-entropy}
\end{equation}
\end{lemma}

\begin{proof}
Let $T=T_\ell(x)$ be the support (Top-K set) of $g=g_\ell(x)$.
Let $u$ be the uniform distribution on $T$, i.e., $u_j=1/K$ for $j\in T$.
Then
\[
\begin{aligned}
Q=\|g\|_2^2
&\overset{1}{=}\|u\|_2^2+\|g-u\|_2^2\\
&=\frac{1}{K}+\|g-u\|_2^2,
\end{aligned}
\]
where the equality 1 holds because:
\[
\begin{aligned}
\|g-u\|_2^2
&=\|g\|_2^2+\|u\|_2^2-2g\cdot u\\
&=\|g\|_2^2+\frac{1}{K}-\frac{2}{K}\\
&=\|g\|_2^2-\|u\|_2^2.
\end{aligned}
\]

Using $\|v\|_2\le \|v\|_1$ and Pinsker's inequality $\|g-u\|_1^2\le 2\mathrm{KL}(g\|u)$,
we have
\[
\|g-u\|_2^2 \le \|g-u\|_1^2 \le 2\mathrm{KL}(g\|u).
\]
Finally, $\mathrm{KL}(g\|u)=\log K - H(g)$ (by direct algebra, see \eqref{eq:entropy-gap-equals-kl}),
which yields \eqref{eq:Q-vs-entropy}.
\end{proof}

\paragraph{Intuition.}
Among distributions on $K$ points, the uniform distribution minimizes collision mass ($Q=1/K$).
Deviations from uniformity increase $Q$, and the entropy gap $\log K-H(g)$ quantifies such deviation.

% \section{Main Theorem: Specialization Controls Random-Pruning Perturbation}
With these lemmas, we can now prove our main theorem.
\begin{theorem}[Expected perturbation upper bound under random pruning]
Fix layer $\ell$ with $N$ experts and Top-$K$ rerouting as defined in Section~\ref{app:pruning_renorm}.
Let $P$ be uniformly random among subsets of size $m\le N-K$, independently of $x$.
Under Assumption~1, we have:
% \begin{equation}
\begin{align}
  &\mathbb{E}_{x,P}\!\left[\|\delta_\ell(x,P)\|_1^2\right]
  \nonumber\\[-2pt]
  &\le 4B_\ell^2\left(
    c_0(m,N)+c_1(m,N)\,\mathbb{E}_x[Q_\ell(x)]
  \right) \nonumber \\
  &\le 4B_\ell^2\left(
    c_0(m,N)+c_1(m,N)
    \left[\frac{1}{K}+2S_\ell\right]
  \right),
  \label{eq:main-bound}
\end{align}
% \end{equation}
with $c_0,c_1$ as in \eqref{eq:c0c1} and $\bar H_\ell=\mathbb{E}_x[H(g_\ell(x))]$.

% \textbf{(Second-moment variant).} Under \eqref{eq:assump-second-moment},
% \begin{equation}
%   \mathbb{E}_{x,P}\!\left[\|\delta_\ell(x,P)\|^2\right]
%   \;\le\;
%   4\Sigma_\ell^2
%   \left(
%     c_0(m,N) + c_1(m,N)\left[\frac{1}{K}+2\bigl(\log K-\bar H_\ell\bigr)\right]
%   \right).
%   \label{eq:main-bound-2nd}
% \end{equation}
\end{theorem}

\begin{proof}
Combine Lemma \ref{lemma:2_1} with Lemma \ref{lemma:3} and Lemma \ref{lemma:4}.
Specifically, for any admissible pruned set $P$, Lemma~\ref{lemma:2_1} gives $\mathbb{E}_{x}[\|\delta_\ell(x,P)\|_1^2] \le 4B_\ell^2\mathbb{E}_{x}[\mathcal{E}_\ell(x,P)^2]$.
Taking expectation over random $P$ conditional on $x$ and applying Lemma \ref{lemma:3} yields dependence on $Q_\ell(x)$.
Finally, Lemma \ref{lemma:4} upper bounds $Q_\ell(x)$ by a function of $\log K - H(g_\ell(x))$, and taking expectation over $x$
gives $\log K - \bar H_\ell$.
\end{proof}

\paragraph{Interpretation.}
For fixed $(m,N,K)$ and layer scale $B_\ell$, the upper bound increases monotonically with
\[
S_\ell=\log K-\bar H_\ell.
\]
Thus, for fixed $(m,N,K,B_\ell)$, a larger entropy gap increases this upper bound, without establishing an ordering of actual perturbations across layers.

\begin{corollary}[Random-to-deterministic implication]
\label{cor:random_to_deterministic}
For a fixed layer $\ell$ and pruning budget $m$, define the deterministic-mask error
\[
\mathcal{E}_\ell(P)
=
\mathbb{E}_{x\sim\mathcal{D}}\!\left[\|\delta_\ell(x,P)\|_1^2\right],
\qquad
P\in\mathcal{P}_{\ell,m}.
\]
Then
\begin{equation}
\begin{aligned}
\min_{P\in\mathcal{P}_{\ell,m}}\mathcal{E}_\ell(P)
&\le \mathbb{E}_{P\sim\mathrm{Unif}(\mathcal{P}_{\ell,m})}
\mathcal{E}_\ell(P)\\
&\le 4B_\ell^2\Bigl(
c_0(m,N)\\
&\qquad+c_1(m,N)
\left[\frac{1}{K}+2S_\ell\right]
\Bigr).
\end{aligned}
\label{eq:deterministic_existence}
\end{equation}
\end{corollary}

\noindent
The minimum over the finite mask family cannot exceed its uniform average. This proves the existence of at least one deterministic mask satisfying the entropy-dependent bound, without identifying a particular selector's output.

\subsection{End-to-End Expert Pruning Error Bound via Discrete Gronwall}
\label{appendix:e2e_bound}

% \paragraph{Assumption A1 (Layer output scale control; same as Assumption 1).}
% For each layer $\ell$, there exists $B_\ell>0$ such that for all $x$ in the considered domain and
% all experts that may appear in Top-$K$ supports before/after pruning,
% \begin{equation}
% \| f_{\ell,j}(x)\|_1 \le B_\ell .
% \end{equation}

% \paragraph{Single-layer bound (recalled).}
% Under uniform random pruning $P_\ell$ among all subsets of size $m_\ell$,
% the prior single-layer theorem gives
% \begin{equation}
% \mathbb{E}_{x\sim\mathcal{D}_\ell,\,P_\ell}\big[\|\delta_\ell(x,P_\ell)\|_1^2\big]
% \;\le\; 4B_\ell^2\left(
% c_0(m_\ell,N_\ell)+c_1(m_\ell,N_\ell)\left[\frac{1}{K}+2(\log K-\bar H_\ell)\right]
% \right),
% \label{eq:single_layer_bound_recall}
% \end{equation}
% where
% \begin{equation}
% c_0(m,N)\triangleq \frac{m(m-1)}{N(N-1)},\qquad
% c_1(m,N)\triangleq \frac{m(N-m)}{N(N-1)}.
% \end{equation}

\subsubsection{Model: Residual Network with MoE Blocks}
Let $h_0 \in \mathbb{R}^d$ denote the input representation. We define a residual network of depth $L$ with residual scaling factor $\alpha > 0$ according to the recursion:
\begin{equation}
h_{\ell+1} = h_\ell + \alpha F_\ell(h_\ell), \quad \ell=0,1,\dots,L-1,
\label{eq:orig_residual}
\end{equation}
where $F_\ell$ represents the layer transition map. We assume $F_\ell$ decomposes into an unpruned component $A_\ell$ (encapsulating operations such as attention and normalization) and a modifiable component $y_\ell$ (e.g., a Mixture-of-Experts branch). Under a pruning policy $P_\ell$, the original and pruned transformations are defined respectively as:
\begin{equation}
\begin{aligned}
F_\ell(h) &= A_\ell(h) + y_\ell(h),\\
\hat F_\ell(h) &= A_\ell(h) + \tilde y_{\ell,P_\ell}(h).
\end{aligned}
\label{eq:F_split}
\end{equation}
Consequently, for any input $h$, the approximation error introduced by pruning is given by:
\begin{equation}
\hat F_\ell(h) - F_\ell(h) = \tilde y_{\ell,P_\ell}(h) - y_\ell(h) \coloneqq \delta_\ell(h,P_\ell).
\label{eq:injection_error_def}
\end{equation}
The dynamics of the pruned network are therefore governed by the recursion:
\begin{equation}
\hat h_{\ell+1} = \hat h_\ell + \alpha,\hat F_\ell(\hat h_\ell), \quad \ell=0,1,\dots,L-1,
\label{eq:pruned_residual}
\end{equation}
initialized with the same state $\hat h_0 = h_0$.

% Let $h_0$ be the input representation.
% We consider a depth-$L$ residual recursion with residual scaling $\alpha>0$:
% \begin{equation}
% h_{\ell+1} = h_\ell + \alpha\,F_\ell(h_\ell),\qquad \ell=0,1,\dots,L-1,
% \label{eq:orig_residual}
% \end{equation}
% where $F_\ell(\cdot)$ is the layer transform.
% We assume that pruning affects only the MoE branch, hence we can write
% \begin{equation}
% F_\ell(h) \;=\; A_\ell(h)+y_\ell(h),\qquad
% \hat F_\ell(h) \;=\; A_\ell(h)+\tilde y_{\ell,P_\ell}(h),
% \label{eq:F_split}
% \end{equation}
% where $A_\ell(\cdot)$ collects all unpruned components in the block (e.g., attention, LN, etc.).
% Therefore, for any input $h$,
% \begin{equation}
% \hat F_\ell(h)-F_\ell(h) \;=\; \tilde y_{\ell,P_\ell}(h)-y_\ell(h) \;=\; \delta_\ell(h,P_\ell).
% \label{eq:injection_error_def}
% \end{equation}

% Define the pruned network recursion
% \begin{equation}
% \hat h_{\ell+1} = \hat h_\ell + \alpha\,\hat F_\ell(\hat h_\ell),\qquad \ell=0,1,\dots,L-1,
% \label{eq:pruned_residual}
% \end{equation}
% with the same initial state $\hat h_0=h_0$.

%-----------------------------------------------------------------------
% Assumption A2: Lipschitz continuity (standardized \ell_1 notation)
%-----------------------------------------------------------------------
\paragraph{Assumption A2 (Local Lipschitz continuity of pruned blocks).}
For each layer $\ell$, assume there is a deterministic constant $\lambda_\ell\ge0$ such that the following inequality holds for every admissible pruning mask. The common domain $\mathcal{D}\subseteq\mathbb{R}^d$ contains both the original and pruned layer-input trajectories for all masks under consideration:
\begin{equation}
\begin{aligned}
    \bigl\|\hat F_\ell(u) - \hat F_\ell(v)\bigr\|_{1}
    &\le \lambda_\ell\,\bigl\|u - v\bigr\|_{1},\\
    &\hspace{1.5cm}\forall\,u, v \in \mathcal{D}.
\end{aligned}
    \label{eq:lipschitz_assump}
\end{equation}

\begin{remark}[Interpretation and practical scope of Assumption~A2]
Assumption~A2 imposes a \emph{local stability} condition: within the relevant trajectory domain $\mathcal{D}$
(e.g., the set of layer inputs reachable under the target data distribution),
the pruned block $\hat F_\ell$ does not amplify input perturbations by more than a factor $\lambda_\ell$
in the $\ell_1$-norm. In our subsequent error-propagation analysis, $\lambda_\ell$ determines the
\emph{propagation factor} $\prod_{k > \ell} (1 + \alpha\lambda_k)$ (or its exponential upper bound),
which quantifies the extent to which upstream perturbations introduced at layer $\ell$ are magnified by subsequent layers.
% Note that for Transformer or MoE blocks employing Top-$K$ routing, global Lipschitz continuity on $\mathbb{R}^d$ generally does not hold;
% we therefore require the inequality only on a bounded or data-induced domain $\mathcal{D}$, consistent with
% standard analyses of local trajectory stability.
\end{remark}

%-----------------------------------------------------------------------
% Assumption A3: Layer-wise independent random pruning (clarified notation)
%-----------------------------------------------------------------------
\paragraph{Assumption A3 (Layer-wise independent random pruning).}
The sequence of pruning masks $\{P_\ell\}_{\ell=0}^{L-1}$ forms a collection of mutually independent
random variables. For each layer $\ell$, the mask $P_\ell$ is uniformly distributed over the set of all
subsets of indices $\{1, \dots, N_\ell\}$ having cardinality $m_\ell$, i.e.,
\begin{equation}
\begin{aligned}
    P_\ell &\sim \mathrm{Unif}\Bigl(
    \bigl\{P \subseteq \{1,\dots,N_\ell\}: |P|=m_\ell\bigr\}
    \Bigr),\\
    &\hspace{1.5cm}\ell=0,1,\dots,L-1.
\end{aligned}
\end{equation}
Furthermore, the entire sequence $\{P_\ell\}_{\ell=0}^{L-1}$ is statistically independent of the
input initialization $h_0$.

\begin{remark}[Role of Assumption~A3 and deterministic pruning]
Assumption~A3 permits closed-form evaluation of the local second moments under uniform random pruning. The error-propagation recursion and discrete-Gronwall bound remain valid for every fixed pruning sequence. Under Assumption~A2, their expected form also applies to an arbitrary mask distribution independent of the evaluation input, as in Eq.~\eqref{eq:e2e_general_selector}; independence between layer masks is not required for that general bound. For a deterministic selector, the explicit entropy form additionally requires Eq.~\eqref{eq:selector_dominance}, giving a local term $\Psi_\ell(m_\ell)+\varepsilon_{\mathcal A,\ell}$. The uniform random pruning theorem below is the special case with local term $\Psi_\ell(m_\ell)$.
\end{remark}

\subsubsection{Key Lemmas: One-step recursion and discrete Gronwall bound}

\begin{lemma}[One-step error decomposition]
\label{lemma:one_step}
Let $\mathcal{E}_\ell \triangleq \|\hat h_\ell - h_\ell\|_1$ denote the $\ell_1$-norm error at layer $\ell$. Under the recursion definitions \eqref{eq:orig_residual}--\eqref{eq:pruned_residual} and Assumption~A2, the error satisfies the following recurrence relation for $\ell=0,\dots,L-1$:
\begin{equation}
E_{\ell+1} \;\le\; (1+\alpha\lambda_\ell)\,\mathcal{E}_\ell \;+\; \alpha\,\|\delta_\ell(h_\ell,P_\ell)\|_1.
\label{eq:one_step_rec}
\end{equation}
\end{lemma}

\begin{proof}
Subtracting \eqref{eq:orig_residual} from \eqref{eq:pruned_residual} and applying the triangle inequality yields:
\begin{align}
E_{\ell+1}
&= \bigl\| (\hat h_\ell + \alpha \hat F_\ell(\hat h_\ell)) - (h_\ell + \alpha F_\ell(h_\ell)) \bigr\|_1 \nonumber \\
&\le \|\hat h_\ell - h_\ell\|_1 + \alpha \bigl\|\hat F_\ell(\hat h_\ell) - F_\ell(h_\ell)\bigr\|_1.
\end{align}
We decompose the functional difference by adding and subtracting $\hat F_\ell(h_\ell)$:
\begin{align}
\bigl\|\hat F_\ell(\hat h_\ell) - F_\ell(h_\ell)\bigr\|_1
&\le \bigl\|\hat F_\ell(\hat h_\ell) - \hat F_\ell(h_\ell)\bigr\|_1 \nonumber\\
&\quad + \bigl\|\hat F_\ell(h_\ell) - F_\ell(h_\ell)\bigr\|_1 \nonumber \\
&\le \lambda_\ell \|\hat h_\ell - h_\ell\|_1 \nonumber\\
&\quad + \|\delta_\ell(h_\ell, P_\ell)\|_1,
\end{align}
where the second inequality follows from the Lipschitz property (Assumption~A2) and the definition of the injection error \eqref{eq:injection_error_def}. Substituting this back into the expression for $E_{\ell+1}$ concludes the proof.
\end{proof}

\begin{lemma}[Global error bound via Discrete Gronwall]
\label{lemma:gronwall}
Suppose $\hat h_0 = h_0$ (i.e., $E_0=0$). Then, for any fixed sequence of pruning masks $\{P_\ell\}$, the final layer error satisfies:
\begin{equation}
\begin{aligned}
E_L &\le \alpha \sum_{\ell=0}^{L-1}
\left( \prod_{k=\ell+1}^{L-1} (1+\alpha\lambda_k) \right)\\
&\qquad\cdot \|\delta_\ell(h_\ell, P_\ell)\|_1.
\end{aligned}
\label{eq:unrolled_gronwall}
\end{equation}
Here, we adopt the convention that the empty product $\prod_{k=L}^{L-1}(\cdot)$ equals $1$.  
\end{lemma}
\emph{Note that the product $\prod_{k=\ell+1}^{L-1}(1+\alpha\lambda_k)$ appears because the perturbation injected at layer $\ell$ is subsequently amplified only by later layers.}

\begin{proof}
We proceed by induction or simply unrolling the recurrence from Lemma~\ref{lemma:one_step}. For $E_0=0$, we have:
\begin{align}
E_1 &\le \alpha\|\delta_0\|_1, \nonumber \\
E_2 &\le (1+\alpha\lambda_1)E_1 + \alpha\|\delta_1\|_1 \nonumber\\
&\le \alpha(1+\alpha\lambda_1)\|\delta_0\|_1
+ \alpha\|\delta_1\|_1.
\end{align}
Iterating this substitution $L$ times yields the summation in \eqref{eq:unrolled_gronwall}. The term $\|\delta_\ell\|_1$ is introduced at step $\ell$ and is subsequently scaled by the expansion factors $(1+\alpha\lambda_k)$ for all future layers $k > \ell$.
\end{proof}

\begin{lemma}[Exponential upper bound]
\label{lemma:exp_bound}
For any non-negative sequence $\{\lambda_k\}$, the product term admits the exponential upper bound:
\begin{equation}
\prod_{k=\ell+1}^{L-1}(1+\alpha\lambda_k) \;\le\; \exp\left(\alpha \sum_{k=\ell+1}^{L-1} \lambda_k \right).
\label{eq:prod_to_exp}
\end{equation}
\end{lemma}

\begin{proof}
The result follows immediately from the elementary inequality $1+x \le e^x$ for $x \ge 0$. Taking the product over $k$ yields $\prod (1+\alpha\lambda_k) \le \prod e^{\alpha\lambda_k} = e^{\alpha \sum \lambda_k}$.
\end{proof}

\subsubsection{Main Theorem: End-to-end Expected Perturbation Bound}

\begin{theorem}[End-to-end perturbation upper bound under layer-wise random pruning]
\label{thm:main_bound}
Suppose Assumptions~A1--A3 hold. Let the network follow the recursion defined in \eqref{eq:orig_residual}--\eqref{eq:pruned_residual}. The expected $\ell_1$-norm perturbation at the final layer satisfies:
\begin{equation}
\begin{aligned}
\mathbb{E}\big[\|\hat h_L - h_L\|_1\big]
&\le \alpha \sum_{\ell=0}^{L-1} w_\ell\\
&\quad\cdot\sqrt{
\mathbb{E}_{x \sim \mathcal{D}_\ell,\, P_\ell}
\big[\|\delta_\ell(x, P_\ell)\|_1^2\big]},
\end{aligned}
\label{eq:e2e_expectation_generic}
\end{equation}
where $w_\ell$ denotes the downstream propagation factor for layer $\ell$:
\begin{equation}
w_\ell \triangleq \prod_{k=\ell+1}^{L-1}(1 + \alpha\lambda_k).
\label{eq:w_def}
\end{equation}
For compactness, define the entropy-dependent local term
\begin{equation}
\begin{aligned}
R_\ell \triangleq{}& c_0(m_\ell,N_\ell)\\
&+c_1(m_\ell,N_\ell)
\left[\frac{1}{K}+2(\log K-\bar H_\ell)\right].
\end{aligned}
\label{eq:local_entropy_term}
\end{equation}
Substituting the single-layer bound \eqref{eq:main-bound} then yields
\begin{equation}
\mathbb{E}\big[\|\hat h_L - h_L\|_1\big]
\le 2\alpha \sum_{\ell=0}^{L-1} w_\ell B_\ell\sqrt{R_\ell}.
\label{eq:e2e_entropy_explicit}
\end{equation}
Finally, applying the exponential upper bound from Lemma~\ref{lemma:exp_bound}, we obtain:
\begin{equation}
\begin{aligned}
\mathbb{E}\big[\|\hat h_L - h_L\|_1\big]
&\le 2\alpha \sum_{\ell=0}^{L-1}
\exp\!\left(\alpha\sum_{k=\ell+1}^{L-1}\lambda_k\right)\\
&\quad\cdot B_\ell\sqrt{R_\ell}.
\end{aligned}
\label{eq:e2e_exp_form}
\end{equation}
\end{theorem}

\begin{proof}
We commence with the unrolled discrete Gronwall bound established in Lemma~\ref{lemma:gronwall} (Eq.~\ref{eq:unrolled_gronwall}):
\begin{equation}
\|\hat h_L - h_L\|_1
\le \alpha \sum_{\ell=0}^{L-1} w_\ell\,\|\delta_\ell(h_\ell, P_\ell)\|_1.
\end{equation}
Taking the expectation with respect to the initialization $h_0 \sim \mathcal{D}_0$ and the sequence of independent pruning masks $\{P_\ell\}$, we have:
\begin{equation}
\mathbb{E}\big[\|\hat h_L - h_L\|_1\big]
\le \alpha \sum_{\ell=0}^{L-1} w_\ell \,\mathbb{E}\big[\|\delta_\ell(h_\ell, P_\ell)\|_1\big].
\label{eq:e2e_step1}
\end{equation}
Since $f(z) = z^2$ is convex, Jensen's inequality implies $\mathbb{E}[X] \le \sqrt{\mathbb{E}[X^2]}$ for any non-negative random variable $X$. Applying this to the error term:
\begin{equation}
\mathbb{E}\big[\|\delta_\ell(h_\ell, P_\ell)\|_1\big]
\le
\sqrt{\mathbb{E}\big[\|\delta_\ell(h_\ell, P_\ell)\|_1^2\big]}.
\label{eq:jensen}
\end{equation}
By definition, the distribution of the layer input $h_\ell$ is denoted by $\mathcal{D}_\ell$. Under Assumption~A3, the mask $P_\ell$ is independent of $h_\ell$ and uniformly distributed. We can therefore rewrite the expectation over the joint distribution:
\begin{equation}
\mathbb{E}\big[\|\delta_\ell(h_\ell, P_\ell)\|_1^2\big]
=
\mathbb{E}_{x \sim \mathcal{D}_\ell,\, P_\ell}\big[\|\delta_\ell(x, P_\ell)\|_1^2\big].
\end{equation}
Substituting \eqref{eq:jensen} into \eqref{eq:e2e_step1} yields the generic bound \eqref{eq:e2e_expectation_generic}.

To derive the explicit forms, we substitute the single-layer result \eqref{eq:main-bound} for the squared error term into \eqref{eq:e2e_expectation_generic}, which directly yields \eqref{eq:e2e_entropy_explicit}. The final exponential form \eqref{eq:e2e_exp_form} follows by applying the inequality $\prod (1+x) \le \exp(\sum x)$ to the propagation weights $w_\ell$.
\end{proof}

\subsection{Corollaries and Design Implications}

\begin{corollary}[Stability under Depth-scaled Initialization]
\label{corr:deepnet_scaling}
Consider a scaling regime where $\alpha = \tfrac{1}{L}$. The error bound \eqref{eq:e2e_exp_form} simplifies to:
\begin{equation}
\begin{aligned}
\mathbb{E}\big[\|\hat h_L - h_L\|_1\big]
&\le \frac{2}{L}\sum_{\ell=0}^{L-1}
\exp\!\left(\frac{1}{L}\sum_{k=\ell+1}^{L-1}\lambda_k\right)\\
&\quad\cdot B_\ell\sqrt{R_\ell}.
\end{aligned}
\end{equation}
If the Lipschitz constants satisfy $\lambda_k \le \lambda_{\max}$, the propagation factor obeys $\exp(\sum_{k=\ell+1}^{L-1}\frac{\lambda_k}{L}) \le e^{\lambda_{\max}}$. Under the stated depth scaling, this factor is bounded independently of $L$.
\end{corollary}

\begin{remark}[Scope of the allocation surrogate]
For large $N_\ell$ at low removal ratios $m_\ell/N_\ell\ll1$, $c_0$ is second order in the removal ratio and $c_1\approx m_\ell/N_\ell$. This approximation does not apply to the main 50\% pruning setting, where both coefficients approach $1/4$ as $N_\ell$ grows. Moreover, the coefficients depend on the removal budget and cannot be discarded when optimizing the bound. SPARE uses routing concentration only as a layer-ranking heuristic:
\begin{equation}
\mathrm{proxy}_\ell \propto S_\ell
% \quad\text{or}\quad
% \mathrm{proxy}_\ell \propto B_\ell^2\,\mathbb{E}_{x \sim \mathcal{D}_\ell}\!\left[\sum_j g_{\ell, j}(x)^2\right].
\end{equation}
The allocation rule does not minimize the summed perturbation bound and omits the layer-dependent output scale $B_\ell$.
\end{remark}

% ===================== Appendix (detailed analysis) =====================
\section{Communication cost analysis}
\label{sec:appendix_ep_comm_savings}

\paragraph{2 nodes $\rightarrow$ 1 node: cross-node EP traffic savings.}
In EP, each MoE layer incurs two all-to-all exchanges (dispatch + combine), whose activation payload is
$2\,Nkd b$ bytes per layer, where $N$ is the number of routed tokens per step (prefill: $N{=}B\!\times\!T$, decode: $N{=}B$),
$d$ is hidden width, $k$ is top-$k$ routing, and $b$ is bytes/element.
When the EP group spans $S{=}2$ nodes and expert ownership is balanced, the expected cross-node fraction is
$p_{\text{cross}} \approx 1-\frac{1}{2}=\frac{1}{2}$.
Therefore, the cross-node all-to-all volume per step across $L_{\text{moe}}$ MoE layers is
\begin{equation}
\begin{aligned}
V_{\text{cross}}(S{=}2)
&\approx (2\,Nkd b)\,L_{\text{moe}}\cdot \tfrac{1}{2}\\
&=Nkd b\,L_{\text{moe}}.
\end{aligned}
\label{eq:cross_node_a2a_s2_simple}
\end{equation}

Plugging DeepSeek-R1-like settings ($k{=}8$, $d{=}7168$, $b{=}2$ bytes, $L_{\text{moe}}{=}58$),
\begin{equation}
\begin{aligned}
V_{\text{cross}}(S{=}2)
&\approx 6{,}651{,}904\,N\ \text{bytes}\\
&\approx 6.65\ \text{MB}\cdot N\\
&\approx 6.34\ \text{MiB}\cdot N\ \text{per step}.
\end{aligned}
\label{eq:cross_node_a2a_numeric}
\end{equation}

If pruning enables consolidating EP from $S{=}2$ to $S{=}1$, then $p_{\text{cross}}=0$ and the cross-node traffic vanishes.
Hence the \emph{saved} cross-node communication is simply
\begin{equation}
\begin{aligned}
\Delta V_{\text{cross}}(2\rightarrow 1)
&\approx V_{\text{cross}}(S{=}2)\\
&\approx 6.34\ \text{MiB}\cdot N\ \text{per step}.
\end{aligned}
\end{equation}
For example, in decode with batch $B$, $N{=}B$ so the saving is $\approx 6.34\,\text{MiB}\times B$ per step.

%%%%%%%%%%%%%%%%%%%%%%%%%%%%%%%%%%%%%%%%%%%%%%%%%%%%%%%%%%%%

\end{document}